\documentclass[journal]{IEEEtran}  

\IEEEoverridecommandlockouts                              

\usepackage{graphicx}
\usepackage{amssymb,amsmath,amsthm} 
\usepackage{algorithm,algpseudocode}
\usepackage[caption=false]{subfig}
\usepackage{booktabs}
\usepackage{tabularx}
\usepackage{bm}
\usepackage{url}
\usepackage{hyperref}
\hypersetup{
    colorlinks=true,
    linkcolor=blue,
    filecolor=blue,      
    urlcolor=blue,
    citecolor=blue,
    pdftitle={Multi-query Robotic Manipulator Task Sequencing with Gromov-Hausdorff Approximations}
    }
\usepackage[dvipsnames]{xcolor}
\usepackage[normalem]{ulem}
\usepackage{balance}
\newtheorem{theorem}{Theorem}

\newtheorem{definition}{Definition}

\newtheorem{problem}{Problem}

\DeclareMathOperator*{\argmin}{arg\,min}

\title{\LARGE \bf
Multi-query Robotic Manipulator Task Sequencing with\\Gromov-Hausdorff Approximations
}

\author{Fouad~Sukkar$^{1*}$,
Jennifer~Wakulicz$^{1}$,
Ki~Myung~Brian~Lee$^{1}$,
Weiming Zhi$^{2}$,
and Robert Fitch$^{1}$
\thanks{*This research is partially supported by the Industrial Transformation Training Centre (ITTC) for Collaborative Robotics in Advanced Manufacturing (also known as the Australian Cobotics Centre) funded by ARC (Project ID: IC200100001). \textit{(Corresponding author: Fouad Sukkar)}}
\thanks{$^{1}$Authors are with the School for Mechanical and Mechatronic Engineering, University of Technology Sydney, 2007, Ultimo, NSW, Australia and *the Australian Cobotics Centre
        {\tt\small {\{Fouad.Sukkar,Jennifer.Wakulicz, KMBrian.Lee,Robert.Fitch\}@uts.edu.au}}}%
\thanks{$^{2}$W. Zhi is affliated with the Robotics Institute, Carnegie Mellon University, Pittsburgh, PA, USA.
        {\tt\small wzhi@andrew.cmu.edu}}%
}

\begin{document}

\maketitle
\thispagestyle{empty}
\pagestyle{empty}

\begin{abstract}

Robotic manipulator applications often require efficient online motion planning. When completing multiple tasks, sequence order and choice of goal configuration can have a drastic impact on planning performance. This is well known as the robot task sequencing problem (RTSP). Existing general-purpose RTSP algorithms are susceptible to producing poor-quality solutions or failing entirely when available computation time is restricted. We propose a new multi-query task sequencing method designed to operate in semi-structured environments with a combination of static and non-static obstacles. Our method intentionally trades off workspace generality for planning efficiency. Given a user-defined task space with static obstacles, we compute a subspace decomposition. The key idea is to establish approximate isometries known as $\epsilon$-Gromov-Hausdorff approximations that identify points that are close to one another in both task and configuration space. 
Importantly, we prove bounded suboptimality guarantees on the lengths of paths within these subspaces. These bounding relations further imply that paths within the same subspace can be smoothly concatenated, which we show is useful for determining efficient task sequences.
We evaluate our method with several kinematic configurations in a complex simulated environment, achieving up to 3x faster motion planning and 5x lower maximum trajectory jerk compared to baselines.

\begin{IEEEkeywords}
Task sequencing, planning, scheduling and coordination, motion and path planning, industrial robots.
\end{IEEEkeywords}

\end{abstract}

\section{INTRODUCTION}

\begin{figure}[!t]
    \centering
    \subfloat[Task-space subspace decomposition]{\includegraphics[width=0.50\linewidth]{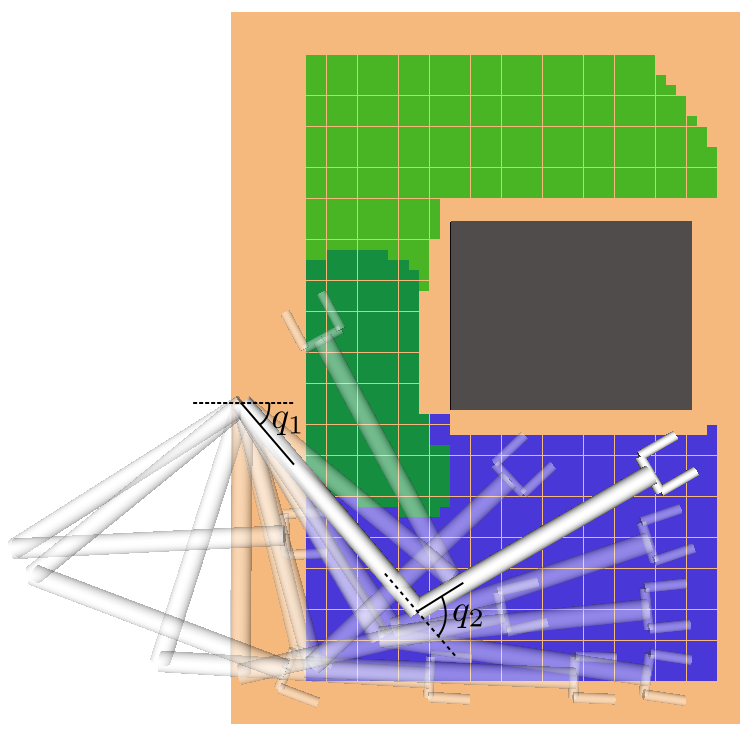}\label{fig:taskspace_subspaces}}
    \subfloat[Mapped subspaces in configuration space]{\includegraphics[width=0.50\linewidth]{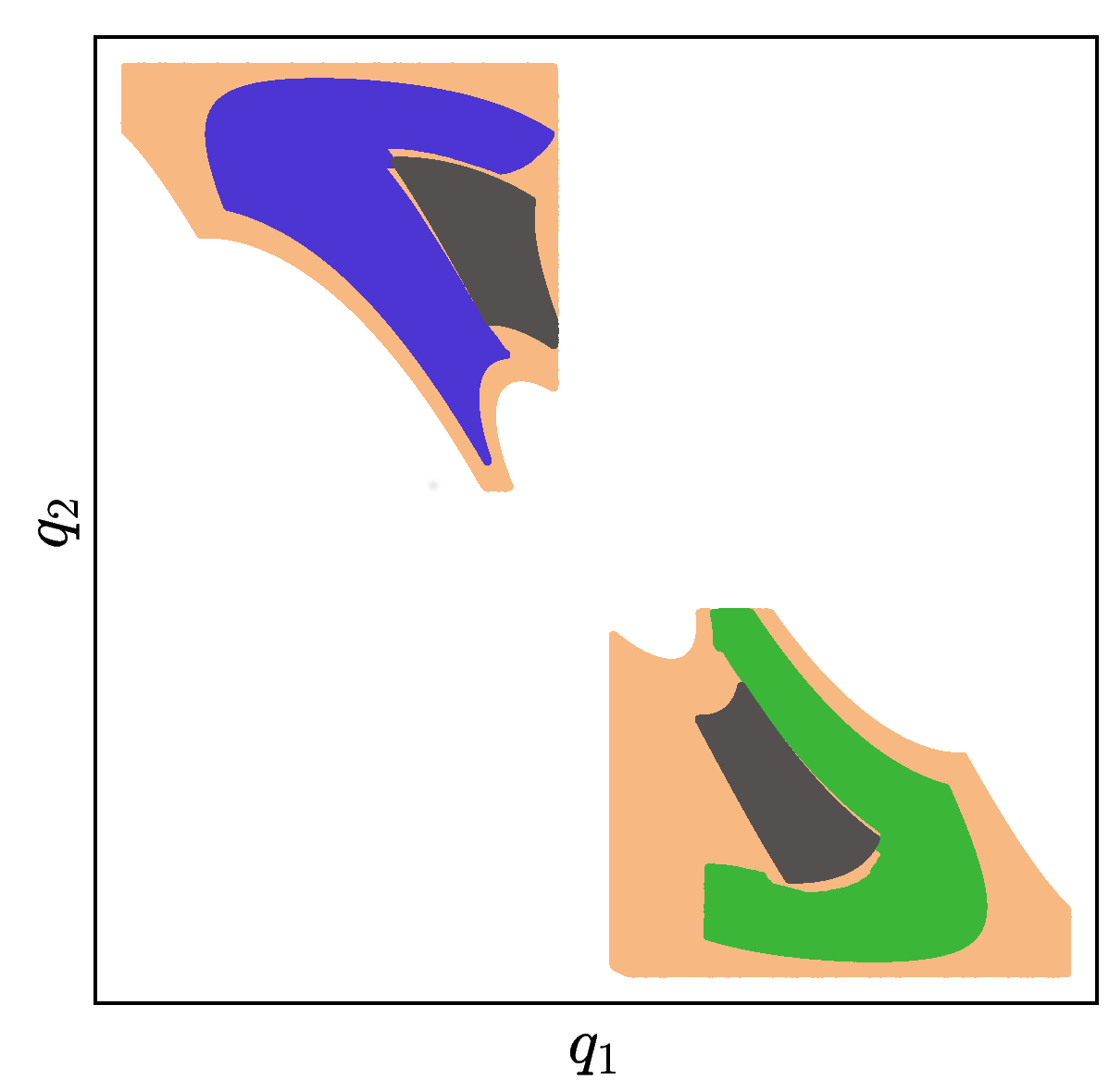}\label{fig:configspace_subspaces}}
    \caption{Example $\epsilon$-Gromov-Hausdorff approximations for a 2-DOF robotic manipulator on a table-top environment with a box obstacle (dark grey). Green and blue regions are subspaces within which end-effector motion does not require large changes in configuration. \protect\subref{fig:taskspace_subspaces}~Task-space subspace decomposition with overlap shown in dark green. \protect\subref{fig:configspace_subspaces}~Mapped disjoint subspaces in configuration space.}
    \label{fig:taskspace_decomp}
\end{figure}

\begin{figure*}[!t]
    \centering
    \subfloat[Naive planner moving from first to second position]{\includegraphics[width=0.33\linewidth]{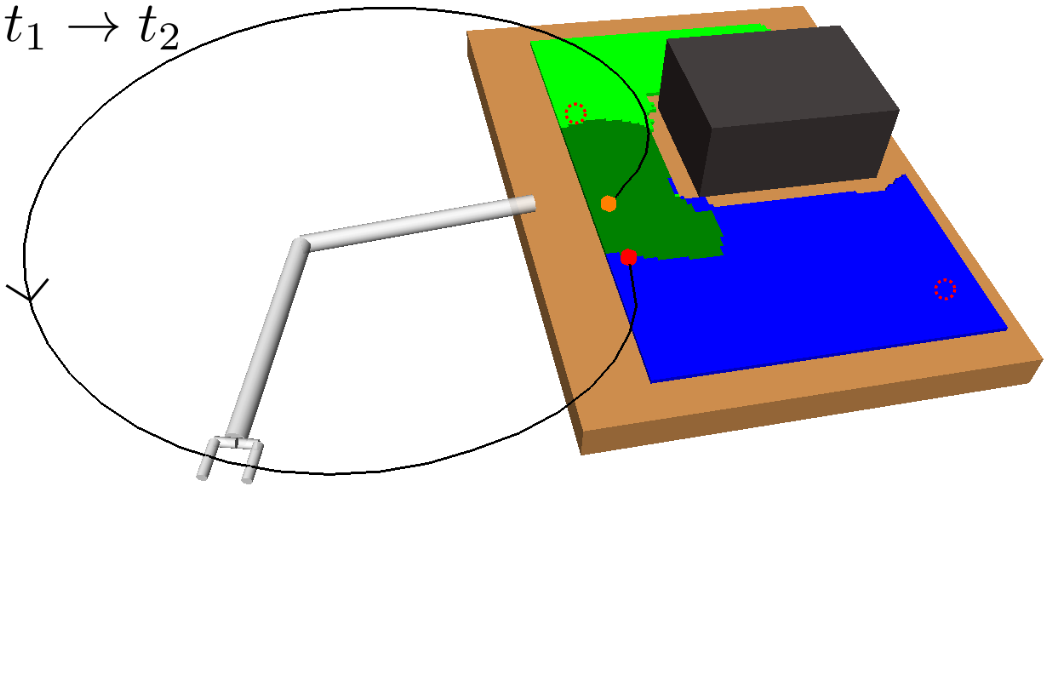}\label{fig:task1_naive}}
    \subfloat[Naive planner moving from second to third position]{\includegraphics[width=0.33\linewidth]{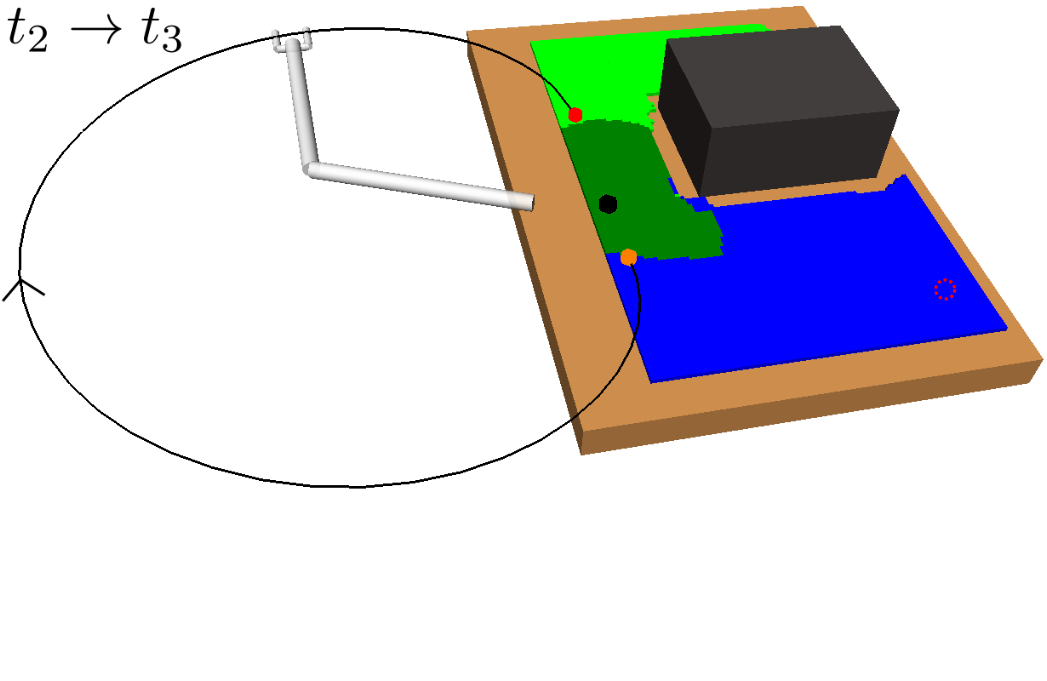}\label{fig:task2_naive}}
    \subfloat[Naive planner moving from third to fourth position]{\includegraphics[width=0.33\linewidth]{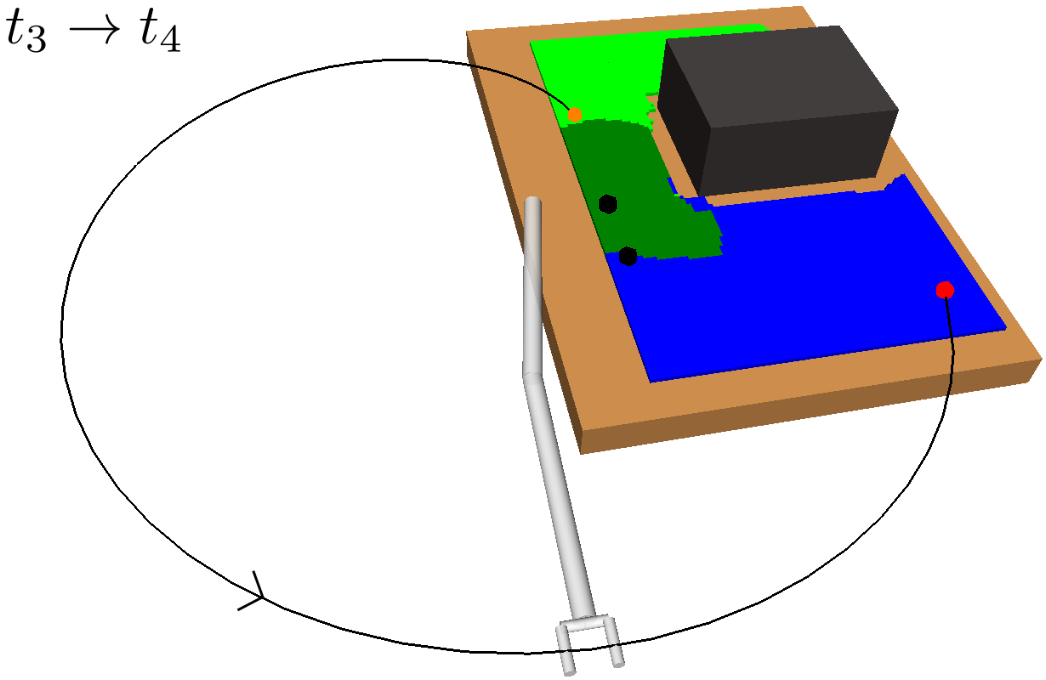}\label{fig:task3_naive}}\\
    \subfloat[HAP moving from first to third position]{\includegraphics[width=0.33\linewidth]{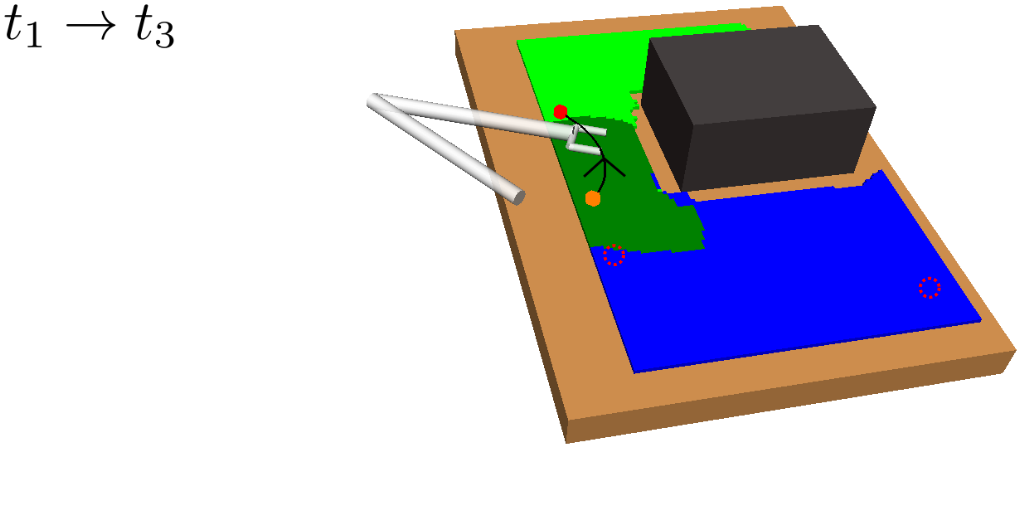}\label{fig:task1_non_naive}}
    \subfloat[HAP moving from third to second position]{\includegraphics[width=0.33\linewidth]{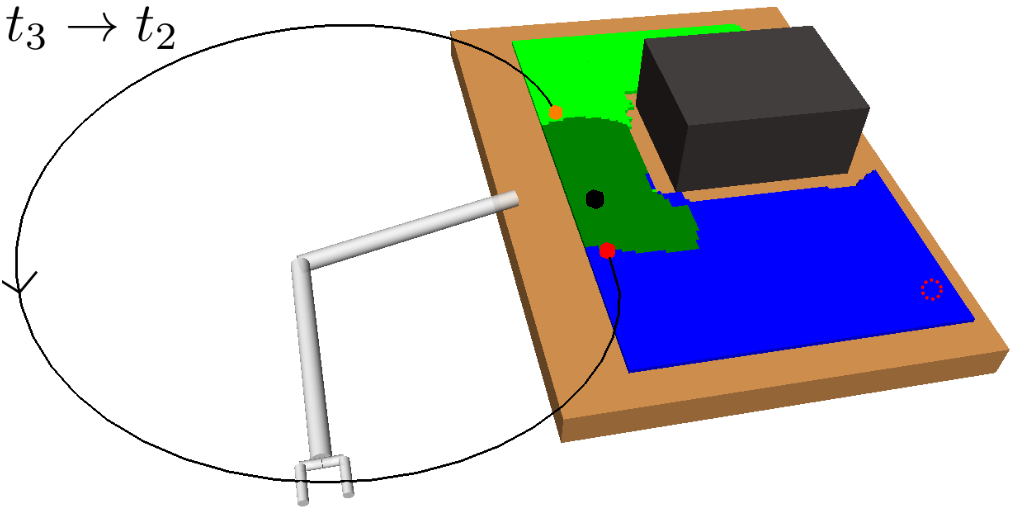}\label{fig:task2_non_naive}}
    \subfloat[HAP moving from second to fourth position]{\includegraphics[width=0.33\linewidth]{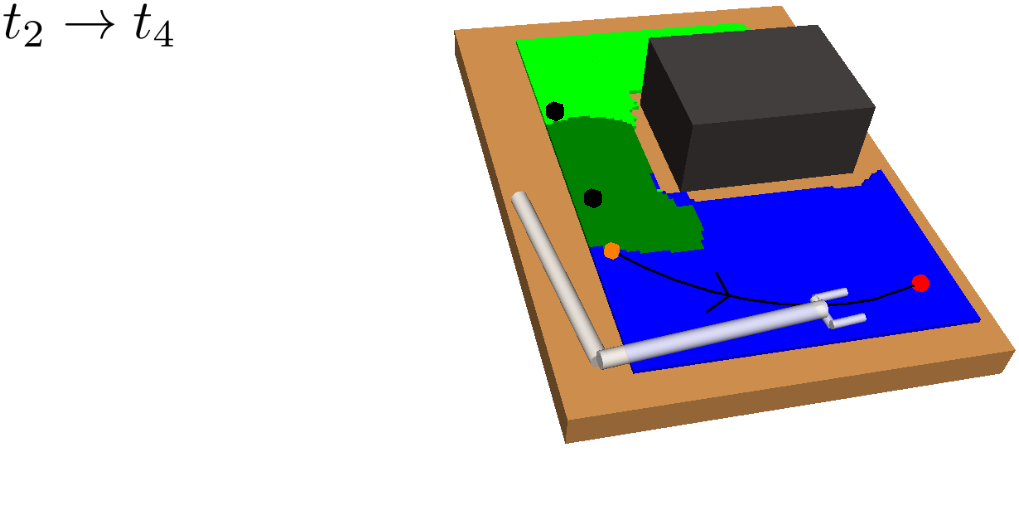}\label{fig:task3_non_naive}}
    \caption{A 2-DOF robotic manipulator tasked with moving its end effector to four unordered positions. Positions are shown as coloured dots. The end-effector path is drawn with direction of movement indicated by the arrowhead. \protect\subref{fig:task1_naive}-\protect\subref{fig:task3_naive} show a sequence of paths produced by a naive planner that only considers task-space distances. \protect\subref{fig:task1_non_naive}-\protect\subref{fig:task3_non_naive} show the sequence produced by our HAP method. HAP's choice of sequencing exploits short within-subspace paths in \protect\subref{fig:task1_non_naive} and  \protect\subref{fig:task3_non_naive}, whereas the naive planner's choice of sequencing underestimates the true motion cost, resulting in longer paths.}
    \label{fig:3tasks}
\end{figure*}

\IEEEPARstart{P}{lanning} algorithms for emerging applications of robotic manipulators must support a greater degree of autonomy than has traditionally been necessary. Robotic manipulators such as \emph{cobots} (collaborative robots), for example, are designed for advanced manufacturing applications where they should operate safely in dynamic work environments shared with humans and should be able to adapt quickly to perform a variety of tasks. These applications share a need for agility and rapid deployment that differs substantially from traditional applications of industrial manipulators, which are typically characterised by repetitive motions in highly structured environments with planning performed completely offline.

Real world applications often require completing a set of tasks; for example, surface inspection, drilling, spray-painting, screw fastening and spot-welding. Determining an efficient task sequence for high-dimensional robot manipulators is challenging due to their kinematic redundancy; that is, a goal pose of the end effector can be achieved via multiple joint configurations. Additionally, the cost induced by the low-level motion between tasks must be reasoned about. Determining an efficient sequence that considers the above is known as the robot task sequencing problem (RTSP)~\cite{alatartsev2015robotic}.

While general-purpose online RTSP algorithms exist, they either decouple the low-level motion~\cite{suarez2018robotsp,li2022efficient} or only partially reason about the low-level motion during sequencing~\cite{wong2020novel}. Furthermore, these algorithms do not offer any performance guarantees. In this paper, we show that this can often lead to highly sub-optimal path sequences and poor motion planning performance. While methods that explicitly consider low-level motion do exist, they are too slow for online settings~\cite{jing2017model}.

We are interested in developing a method that efficiently solves the RTSP for practical situations that require repeated, rapid, and reliable planning, where major structures in the environment are static. A common example is a manipulator that must inspect and grasp objects from shelves and cabinets in a domestic or warehouse environment~\cite{srinivasa2012herb,dogar2013physics,morrison2018cartman}. The shelves are static and their dimensions can be measured beforehand, while the objects on the shelves might not be known and their locations can change. 

Multi-query planners such as the probabilistic roadmap~(PRM)~\cite{kavraki1996probabilistic} aim to gain efficiency through computational reuse; a computationally costly offline process generates a data structure that can then be repeatedly queried efficiently by a low-cost online process. 
Inspired by PRMs, in this work, we seek a data structure that facilitates fast online task sequencing whilst providing some practical path quality guarantees.

Critically, we identify a particularly desirable property of motion plans: \textbf{if two coordinates are close in task space, a smooth, short path exists between them in configuration space.} We propose to decompose the task space into subspaces, such that paths through a particular subspace satisfy this property. We achieve this by establishing a bounding relation between distances in task space and corresponding distances in configuration space. Such a mapping between metric spaces is known as an $\epsilon$-Gromov-Hausdorff approximation~($\epsilon$-GHA).

We present the \emph{Hausdorff approximation planner~(HAP)}\footnote[1]{\tt Open-source implementation of HAP available at: \url{https://github.com/UTS-RI/HAP-py}} which, given a task space and obstacle configuration, computes a set of covering $\epsilon$-GHAs. These mappings provide a compact structure to efficiently plan over and additionally disambiguate kinematic redundancy by providing a unique mapping from task space to configuration space. This is key to our approach for efficiently solving RTSPs where multiple tasks can be clustered based on their associated subspaces, resulting in a set of reduced sequencing problems that can be solved independently.

A motivating example of this subspace decomposition for a 2-DOF manipulator is shown in Fig.~\ref{fig:taskspace_decomp}. Here, the two subspaces overlap in task space. However, when mapped to configuration space via an $\epsilon$-GHA two disjoint subspaces are revealed. Fig.~\ref{fig:3tasks} illustrates a scenario where this mapping is useful for planning efficient path sequences that remain in a single subspace. As can be seen in the figure, a naive greedy approach that only considers task-space distances results in unnecessary wasted motion as opposed to our method which utilises the subspace knowledge.
In higher dimensions, these subspaces may be less obvious to choose from. For example, for 6-DOF or 7-DOF robot arms, several of these subspaces exist. Examples include, wrist/elbow/shoulder-in versus wrist/elbow/shoulder-out configurations and even permutations of these resulting from multiple possible revolutions around each joint.

We evaluate our method in a complex simulated bookshelf environment where a varying number of tasks must be completed in an efficient order. To highlight HAP's applicability to a variety of kinematic configurations we perform experiments with 6-DOF and 7-DOF manipulators and with a manipulator mounted to a mobile base. Compared to state-of-the-art RTSP methods, our method achieves up to 3x faster motion planning and 5x lower maximum trajectory jerk with consistently higher planning success, even when unknown objects are added to the scene.

In summary, the main contributions of this paper are as follows. 1) The HAP framework, a multi-query method for solving RTSPs which, in contrast to previous work, explicitly considers the low-level motion of the robot. 2) A bounded suboptimality analysis which provides theoretical guarantees on the quality of paths through the $\epsilon$-GHAs computed by HAP. 3) An extensive evaluation of HAP against state-of-the-art baselines.

\section{RELATED WORK}
In this section we describe and analyse related work in the literature with respect to our approach. We begin by exploring the origins of the robot task sequencing problem and discuss its development in recent years. We conclude the section with a brief discussion on relevant work in motion planning and useful properties of our framework in such contexts.

\subsection{Robot Task Sequencing}

The RTSP is a well-studied problem, and approaches can generally be divided into offline and online approaches. Offline approaches tend to formulate the problem as a generalised TSP (GTSP)~\cite{noon1993efficient} (also known as the set TSP) where task inverse kinematics~(IK) solutions are formed into subsets and the shortest tour between subsets must be found~\cite{alatartsev2015robotic}. Determining an optimal sequence generally takes several minutes to hours for 6-DOF and higher robot arms, even when ignoring the cost of the low-level motion of the robot and are thus unsuitable for online planning~\cite{zacharia2005optimal,kolakowska2014constraint, villumsen2017framework}.

Online methods sacrifice sequence optimality for faster planning. For example, if the kinematic redundancy of the arm is ignored and only task-space distances are considered, then the sequencing problem can be treated as a TSP~\cite{apple_2017,bac2017performance}. While this results in much faster task sequencing, task-space costs can heavily underestimate the true cost of moving between tasks due to the non-linear nature of high-dimensional robot arms~\cite{You2020}.

Decoupled approaches overcome this limitation by first computing a sequence in task space and then determining an optimal goal configuration assignment.
For example, RoboTSP~\cite{suarez2018robotsp} initially computes a sequence using task-space distance and then assigns optimal joint configurations to each task using a graph search algorithm. The work in~\cite{li2022efficient} adds an extra filtering step to RoboTSP which removes dissimilar candidate IK solutions from the graph search which results in faster task sequencing yet similar execution times. While these decoupled approaches are an improvement over task-space sequencing alone, the choice of configuration assignment is dependent on the initial task sequence and may still produce highly suboptimal trajectories.

More recently, Cluster-RTSP~\cite{wong2020novel} was proposed to overcome the shortcomings of decoupled approaches by directly sequencing in configuration space. To make the problem tractable, unique configurations are first assigned to each task based on a best-fit heuristic. These configurations are then clustered into groups and a sequence is found by solving inter-cluster and intra-cluster TSPs individually. While this method achieved a large reduction in task execution and computation times compared to RoboTSP, the travel costs used for sequencing were still Euclidean distance between assigned task configurations which may not effectively capture costs incurred by the true low-level motion, for example, due to obstacle avoidance. Furthermore, outlier tasks with highly dissimilar IK solutions can heavily impact the overall quality of the task sequence. 

Learning-based approaches have been proposed to solve RTSPs. For example, in~\cite{dong2023optimizing} the problem was framed as a Markov decision process and a reinforcement learning algorithm learned a policy that outputs a time and energy efficient task sequence and corresponding goal IK solution for a robot spot welding task. While results show an overall reduction in computation time and trajectory efficiency compared to offline methods, they assume no obstacles between tasks. Given the black-box nature of learning-based approaches it is unclear how well they generalise to the more complex cluttered environments we consider in this work.

The Mobile Manipulator RTSP extends this problem to mobile manipulators. The work in~\cite{nguyen2023task} approaches this problem by clustering tasks based on the arm's reachability from a finite set of base poses. A minimum set of base poses is then chosen and RoboTSP is used to sequence each cluster. We show that $\epsilon$-GHAs can be naturally extended to mobile bases and enhance subspace allocation, allowing for greater flexibility in subspace assignments in addition to operating in potentially larger workspaces.

While RTSP methods exist that explicitly consider trajectory costs, they either simplify the problem or are intractable for anything other than a small number of tasks. For example, in~\cite{spitz2000multiple} configurations are arbitrarily chosen for each task and motion plans and generated between them. In~\cite{jing2017model} the RTSP was formulated as a combined set covering problem and TSP (SCTSP) and trajectories were exhaustively computed between all possible goal configuration assignments which took 3 hours for 400 tasks. 
In~\cite{zacharia2013task} a genetic algorithm was used to find an optimal sequence that considered obstacle avoidance in its optimisation objective. However, problems with only 15 tasks were considered and computation times were not reported.
The work in~\cite{kovacs2016integrated} formulated a robotic laser welding problem as a TSP with neighborhoods (TSPN) and was given a computation budget of 600 seconds for up to 71 tasks; however, obstacles were ignored.
In contrast, HAP computes a task-space decomposition that considers robot kinematics and prior obstacle knowledge in a reasonable amount of time and facilitates rapid online multi-query task sequencing with high-quality trajectories.

\subsection{Motion Planning}

Our framework has useful properties for motion planning applications. As we will show, trajectories contained within $\epsilon$-GHAs are useful as seed trajectories for trajectory optimisation methods~\cite{zucker2013chomp}. These methods are well-known to be sensitive to the initial seed trajectory and susceptible to failing due to local minima~\cite{schulman2014motion}. Trajectories contained within $\epsilon$-GHAs mitigate this issue since they consider manipulator kinematics and \emph{a priori} knowledge of the environment. 

Work in integrated task and motion planning~\cite{garrett2021integrated} typically utilise some decomposition of the workspace for problem reduction. For example, in~\cite{simeon2004manipulation} authors characterise a lower-dimensional subspace of valid grasp and place configurations. A roadmap is then constructed over these subspaces to search for transit/ transfer modes. However, the construction of these subspaces is mainly concerned with feasibility rather than efficiency.

Additionally, other aspects of our approach share similarities with trajectory library methods~\cite{ellekilde2013motion,lee2014fast,lin2016using}, workspace decomposition~\cite{reid2020sampling,chamzas2021learning} and manipulator coverage planning~\cite{hassan2017coverage,paus2017coverage}. However, our approach differs in that the task-space decomposition and paths through them are constructed specifically in a way to facilitate efficient robot manipulator task sequencing. Furthermore, we provide practical guarantees on the lengths of paths through the decomposition which leads to predictable and well-defined behaviour.

\section{PROBLEM FORMULATION}
Here we provide necessary notation and problem definitions for describing the RTSP. We then provide an overview of our approach which reduces the RTSP complexity via subspace knowledge.

\begin{figure}[!t]
    \centering
    
    \subfloat[Classic TSP]{\includegraphics[width=0.45\linewidth]{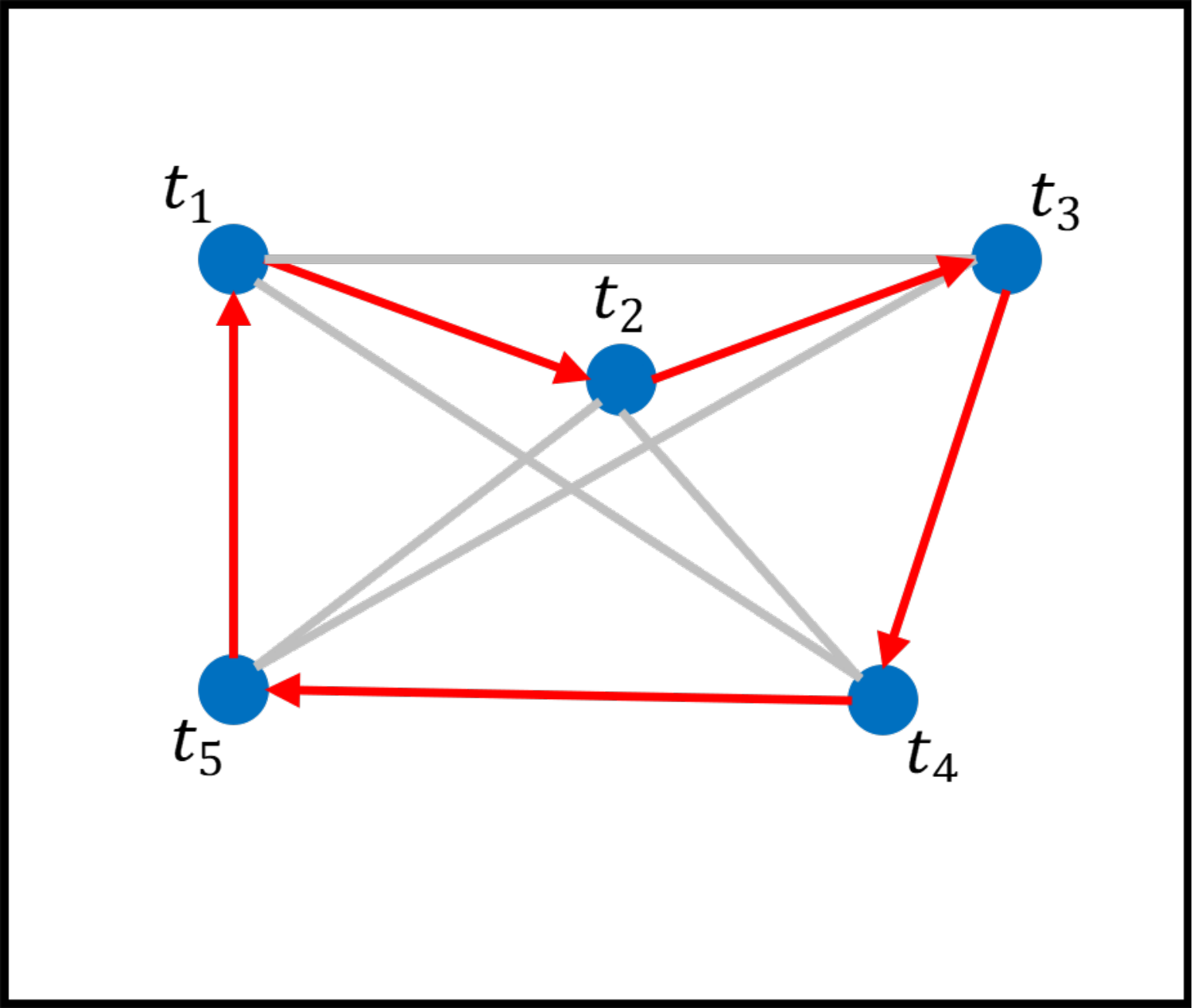}\label{fig:classic_tsp}} \hfill
    \subfloat[Generalised TSP (GTSP)]{\includegraphics[width=0.45\linewidth]{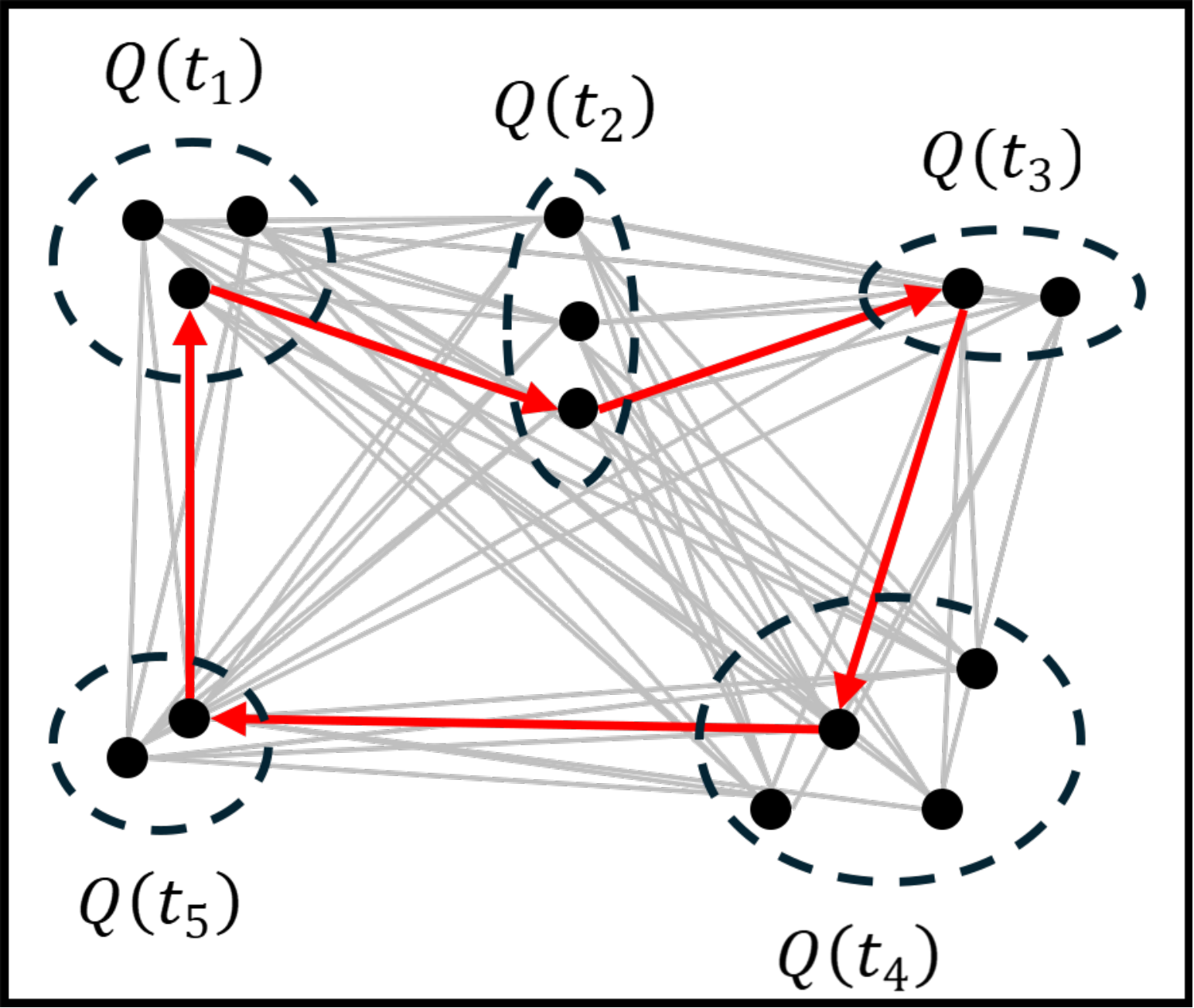}\label{fig:set_tsp}}\\
    \subfloat[Reduced problem with $\epsilon$-GHAs]{\includegraphics[width=0.9\linewidth]{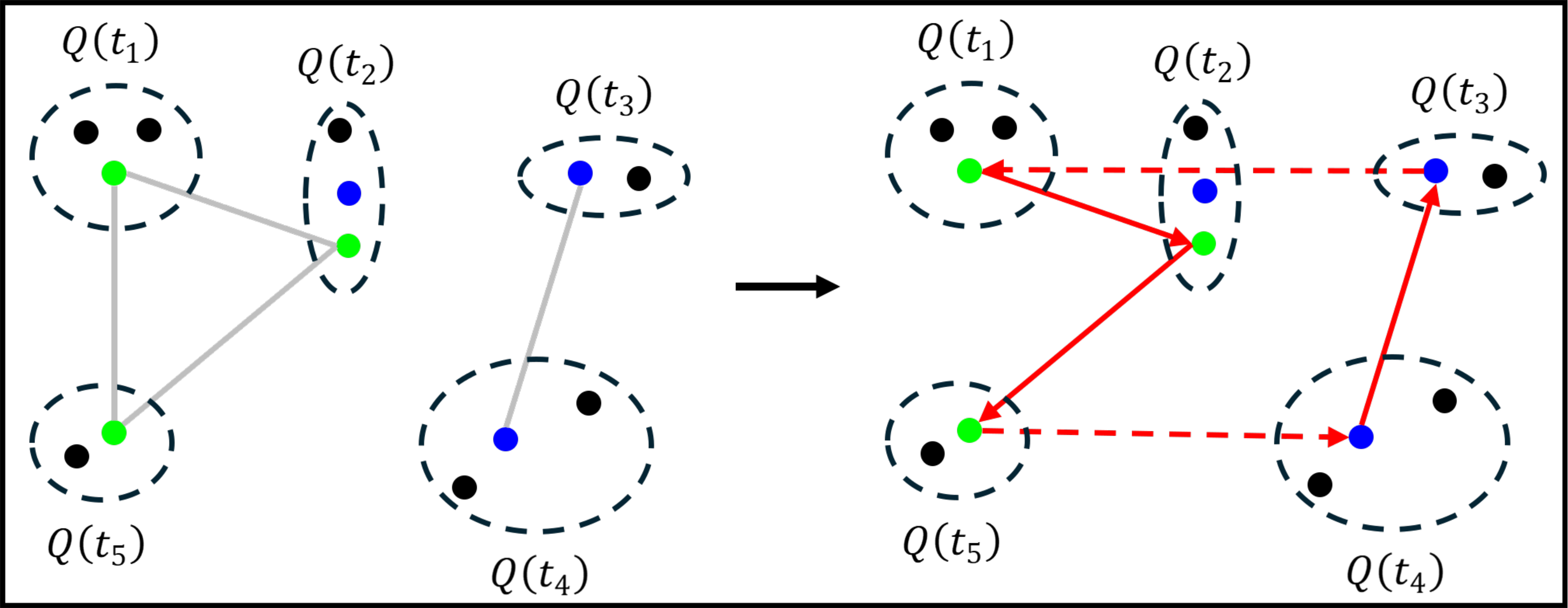}\label{fig:hap_tsp}}
    \caption{Overview of problem and approach. \protect\subref{fig:classic_tsp}~Classic TSP ignores configuration space path costs (nodes are task poses). \protect\subref{fig:set_tsp}~GTSP considers all possible path sequences (node sets represent task pose IK solutions). \protect\subref{fig:hap_tsp} Reduced problem (ours) utilises $\epsilon$-GHAs, only considers subset of IK solutions (green and blue nodes belong to different subspaces) and solves individual TSPs based on subspace assignment.}
    \label{fig:prob_overview}
\end{figure}

\subsection{Notation}
$\mathcal{C}$ represents the configuration space of the robot. In this work we target serial manipulators. The workspace, $W$, is the 3D Euclidean workspace, $W=\mathbb{R}^3$. Given a configuration $q \in \mathcal{C}$, $A(q) \subset W$ denotes the space occupied by the robot model at configuration $q$. $\hat{m}$ $ \subset \mathbb{R}^{3}$ is an approximate model of the environmental obstacles. We assume access to a collision checking process that reports whether the arm is in collision with $\hat{m}$ or with itself. The obstacle region is defined as $\mathcal{C}_{\text{obs}} = \{q \in \mathcal{C} \mid A(q)\cap $$\hat{m}$ $ \neq \emptyset\}$, from which we obtain the free space region $\mathcal{C}_{\text{free}} = \mathcal{C}\setminus \mathcal{C}_{\text{obs}}$. A task, modelled as a 6D pose $t$, typically has a set of inverse kinematics~(IK) solutions $\Xi(t)$. Then, the task space $\mathcal{T} \subset SE(3)$ is the set of poses of the robot's end effector for which \textit{valid} IK solutions, $Q(t) \subset$ $\Xi(t)$ exist. Any IK solution $q$ is considered valid if $q \in \mathcal{C}_{\text{free}}$.
$\hat{T}$ is a discrete approximation of the subset of $\mathcal{T}$ where the robot is expected to operate frequently. Summary of notation can be found in Table~\ref{table:notation}.

\subsection{Robot Task Sequencing}
To complete a given task $t \in \mathcal{T}$ chosen from the task space $SE(3)$ the manipulator must position its end effector at pose $t$ while avoiding collision. 
We are interested in finding a minimum cost path in $\mathcal{C}_{\text{free}}$ that completes task $t$. The manipulator's path is modelled as a discrete sequence of configurations $\pi = \{\pi[1],\ldots,\pi[N] \mid \pi[i] \in \mathcal{C}_{\text{free}}\}$. We measure the length of a path using a metric on the configuration space,
\begin{equation}
d_{C}(\pi) = \sum_{n=1}^{N-1} d_{C}( \pi[n], \pi[n+1] ).
\end{equation}

It is convenient to consider the starting pose of the end effector to be the goal pose of the previous task. A manipulator may be able to achieve the goal pose with multiple, possibly infinite, configurations. We are therefore interested in the set of paths (sequence of configurations) of varying length $N$ between task $t_j$ and $t_l$, $\Pi(t_{j},t_{l}) = \{\pi \mid \pi[1] \in Q(t_{j}), \pi[N] \in Q(t_{l})\}$,  leading to the following problem definition.

\begin{problem} [Manipulator goal configuration assignment problem] \label{prob:motion_plan}
Find a goal configuration that achieves the shortest collision-free path $\pi^{*}$ in configuration space between two tasks $t_j$ and $t_l$ in $\mathcal{T}$,

\begin{equation}\label{eq:motion_planning}
    \pi^{*}(t_{j}, t_{l}) = \argmin_{\pi \in \Pi(t_{j}, t_{l})} d_{C}(\pi).
\end{equation}
\end{problem}

The operational scenarios that motivate this work often involve more than one task. The manipulator is given an unordered set of $M$ tasks $T = \{t_{i}\}_{i=1}^M \subset \mathcal{T}$. This set of tasks can be viewed as a batch-query scenario where all tasks must be completed while minimising total cost. Thus, it is necessary to choose a sequence of tasks, imposing a total ordering over $T$, in addition to repeatedly solving Problem~\ref{prob:motion_plan}. The robot task sequencing problem can thus be formulated as follows.

\begin{problem} [Robot task sequencing problem] \label{prob:task_seq}
Find a permutation of the tasks $\sigma \in S_{M}$, where $S_M$ is the set of possible task sequences, that minimizes the total path cost:
\begin{equation}\label{eq:task_sequencing}
    \min_{\sigma \in S_{M}} \sum_{n=1}^{M-1} d_{C}(\pi^{*}( t_{\sigma[n]}, t_{\sigma[n+1]} )), 
\end{equation}
such that the configuration at the end of $\pi^*(t_{\sigma[n]}, t_{\sigma[n+1]})$ and the start of $\pi^*(t_{\sigma[n+1]}, t_{\sigma[n+2]})$ are equal for all $n \in [1, M-2]$.
\end{problem}

Problem~\ref{prob:task_seq} does not admit a tractable solution. For each task pose $t_{j} \in T$ to be visited, there can be multiple valid configurations in set $Q(t_{j})$\footnote{For infinite sets, as is the case for redundant arms, we produce a finite set by discretising around one of the free joints}. A direct solution requires simultaneously choosing the optimal configuration from $Q(t_j)$ and sequencing the tasks. In other words, this case contains an instance of the generalised TSP~(GTSP)~\cite{noon1993efficient} (Fig.~\ref{fig:set_tsp}), which is even more complex than the standard TSP~\cite{matai2010traveling} (Fig.~\ref{fig:classic_tsp}), and does not allow for real-time planning. The number of possible sequences to evaluate is $\mathcal{O}(|Q(t_j)|^{M} \times M!)$, where $|Q(t_j)|$ is the cardinality of the largest set of IK solutions for any task $t_j \in T$. Furthermore, evaluating the path cost for each configuration pair requires solving a motion planning problem, which is itself PSPACE-complete~\cite{canny1988complexity}. 
These computational complexities, which we wish to avoid, motivate our approach to solving Problem~\ref{prob:task_seq}.

\subsection{Approach Overview}

Reducing the set of IK solutions of each task $t_j$ to a unique solution $q_{t_j}\in Q(t_j)$ transforms the GTSP in Problem~\ref{prob:task_seq} to a classical TSP, which is easier to solve. We therefore search for a mapping $\theta$ from each task $t_{j}$ in a discrete approximation $\hat{T}$ of the full task space $\mathcal{T}$ to a suitable unique IK solution $\theta(t_{j}) \in Q(t_j)$.

To ensure that task sequencing using this mapping guarantees short, smooth and collision-free paths between tasks in both task and configuration space, the map $\theta : \hat{T} \to \mathcal{C}$ is chosen to be an approximate isometry. Intuitively, an approximate isometry enforces that two tasks close in task space remain close in configuration space after mapping. Once $\theta$ is found, the GTSP in Problem~\ref{prob:task_seq} is successfully reduced to the classical TSP problem as desired, and task sequencing can be solved efficiently. As a result, the solution of the reduced problem is a near-optimal solution to Problem~\ref{prob:motion_plan}, producing smooth, short paths between tasks.

However, the kinematics of a robotic manipulator working in the task space may not necessarily admit a single approximate isometry $\theta$. For example, in Fig.~\ref{fig:task1_naive} there is no single $\theta$ that allows the arm to travel the short task-space distance required without taking a long path through configuration space. To handle such cases, we relax the constraint of enforcing a unique $\theta$ and consider a finite set of maps. Thus, we propose finding $K$ mappings $\{\theta\}_{t=1}^K$ from $K$ subspaces in task space to distinct subspaces in configuration space, i.e., $\theta^i:\hat{T}^{i}\rightarrow\mathcal{C}^i$. Within each subspace, the corresponding approximate isometry $\theta^i$ guarantees short, smooth and collision-free paths in both task and configuration space. 

This effectively decomposes the task space into subspaces that may overlap, leading to cases where a single task is assigned multiple IK solutions rather than our single desired one. As such, the GTSP problem is no longer reduced immediately to a classical TSP one.
To solve Problem~\ref{prob:task_seq} we must first disambiguate which IK solution to select in overlapping subspace regions. Then, we may solve a classical TSP problem within each subspace independently. The final reduced problem and approach is illustrated in Fig.~\ref{fig:hap_tsp}.

\section{HAUSDORFF APPROXIMATION PLANNER FRAMEWORK}

\begin{figure}[!t]
    \centering
    \includegraphics[trim={3.5cm 0.9cm 2.0cm 0.9cm},clip,width=1.0\linewidth]{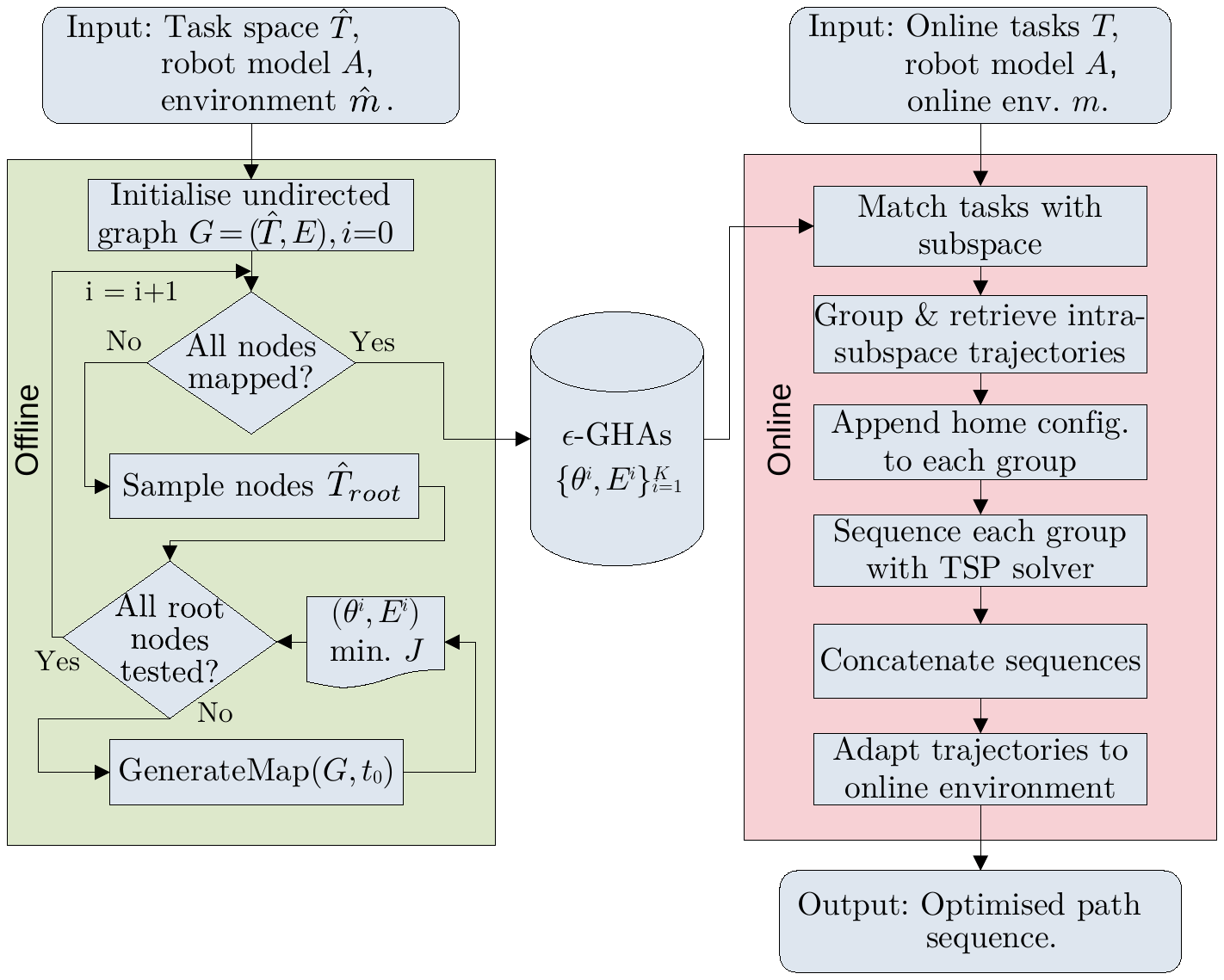}\label{fig:flow_chart}
    \caption{Overall process of the HAP framework. On the left is the task decomposition which generates the $\epsilon$-GHAs offline. On the right describes the task-sequencing process given an online scenario which utilises the $\epsilon$-GHAs.}
    \label{fig:flowchart}
\end{figure}

\begin{figure*}[!ht]
    \centering
    \subfloat[]{\includegraphics[width=0.49\columnwidth]{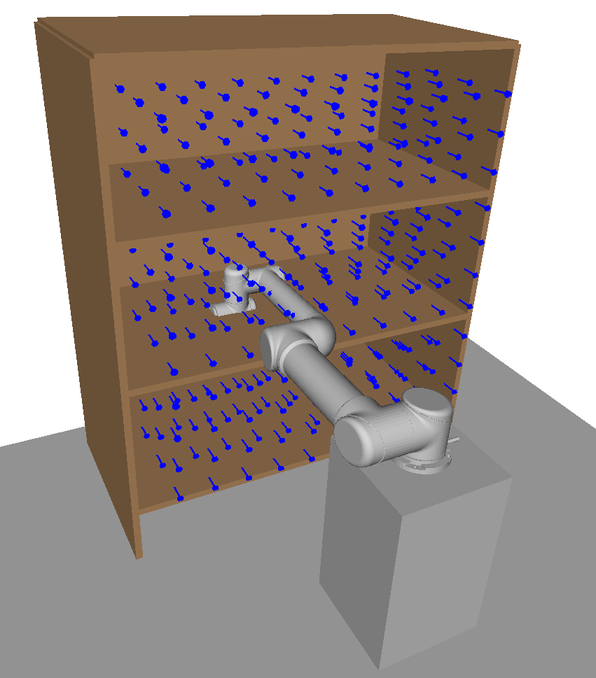}%
    \label{fig:offline_grid}}
    \hfil
    \subfloat[]
    {\includegraphics[width=0.499\columnwidth]{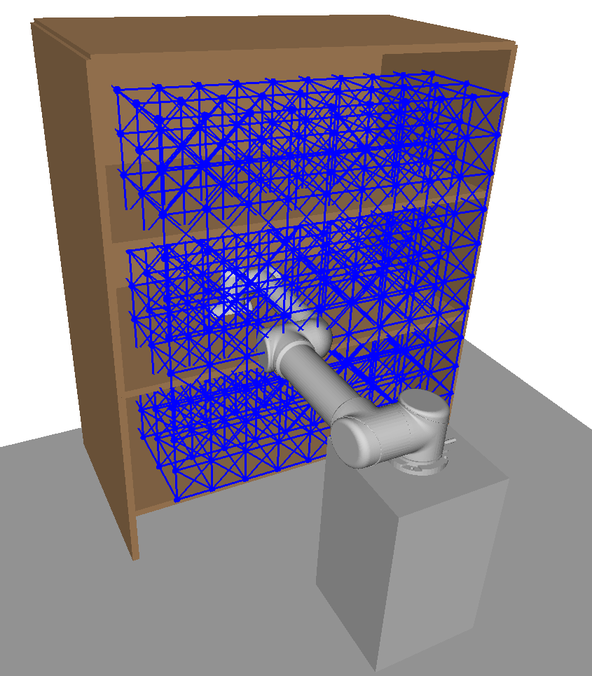}\label{fig:offline_edges}}%
    \hfil
    \subfloat[]
    {\includegraphics[width=0.499\columnwidth]{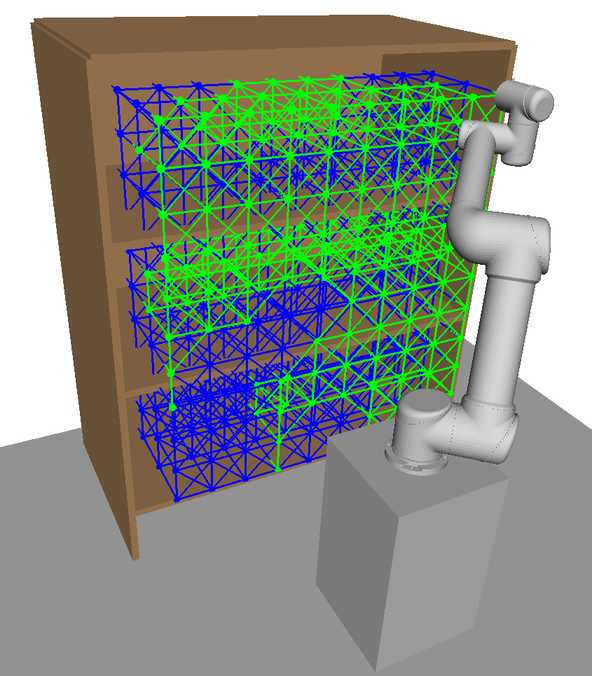}\label{fig:offline_edges_subspace}}%
    \hfil
    \subfloat[]{\includegraphics[width=0.49\columnwidth]{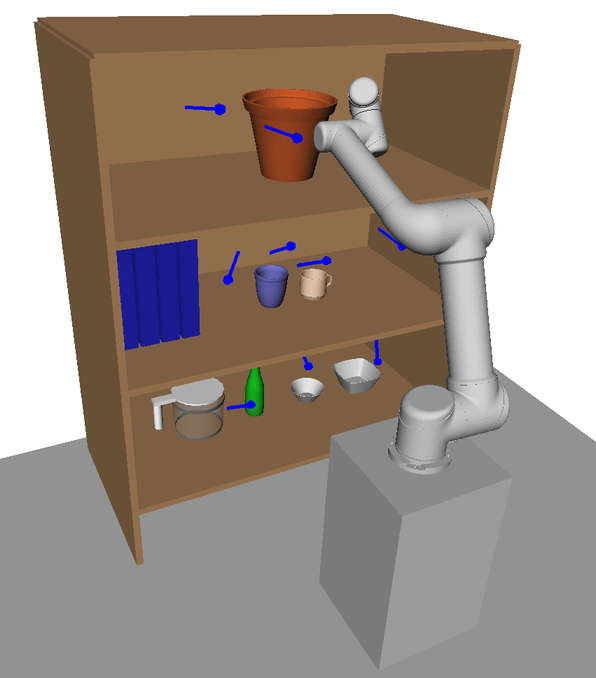}\label{fig:online_scenario}}%
    \caption{Example HAP task space decomposition and online scenario.
    \protect\subref{fig:offline_grid} anticipated discretised task space $\hat{T}$ (blue poses) and environment $\hat{m}$. \protect\subref{fig:offline_edges} an undirected graph $G$ constructed over $\hat{T}$. \protect\subref{fig:offline_edges_subspace} a single $\epsilon$-GHA mapped in task-space represented as a subgraph of $G$, subspaces are searched for by traversing the graph and assigning an unique IK solution to each node such that all connected neighbours are close in configuration space. \protect\subref{fig:online_scenario} an example online scenario with the arm in its home configuration. Online, HAP is robust against objects in \protect\subref{fig:online_scenario} that are unmodelled in $\hat{m}$ and tasks can differ to those in $\hat{T}$.}
    \label{fig:hap_overview}
\end{figure*}

In this section we detail the HAP framework which consists of an offline pre-computation stage, and an online planning stage, see Fig.~\ref{fig:flowchart}. In the pre-computation stage the map(s) $\{\theta^i\}_{i=1}^K$ and corresponding task space decompositions are generated as detailed in Sections~\ref{sec:task-space_decomp}~and~\ref{sec:egha_comp}. Then, during online planning the online tasks are matched to the mapped subspaces, intra-subspace TSPs are solved independently, and the final motion is generated for execution. These online stages are described in Section~\ref{sec:task_seq}.

\subsection{Task-space Decomposition}\label{sec:task-space_decomp}

\begin{figure*}[!ht]
    \subfloat[]{\includegraphics[trim={2cm 0cm 2cm 0cm},clip,width=0.259\columnwidth]{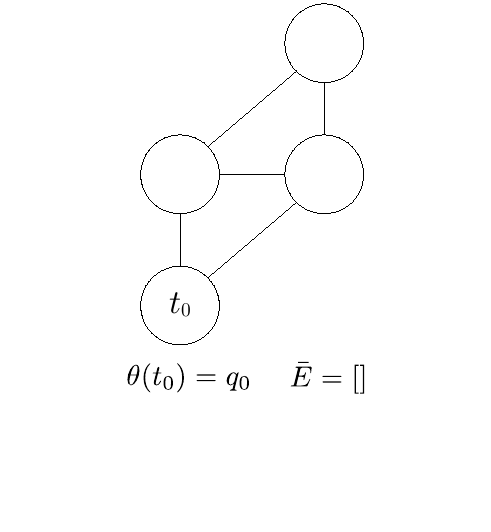}%
    \label{fig:egha_vis1}}
    \hfil
    \subfloat[]
    {\includegraphics[width=0.499\columnwidth]{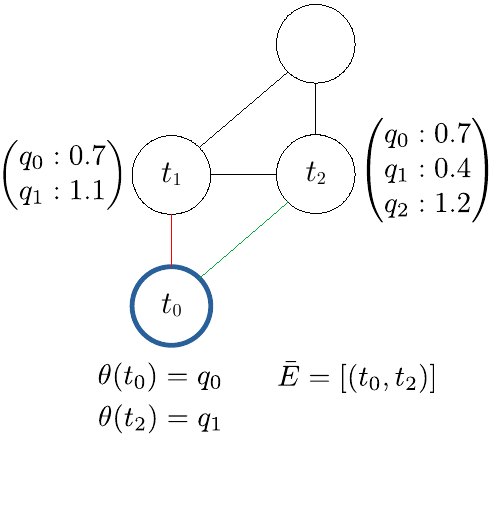}\label{fig:egha_vis2}}%
    \hfil
    \subfloat[]
    {\includegraphics[width=0.499\columnwidth]{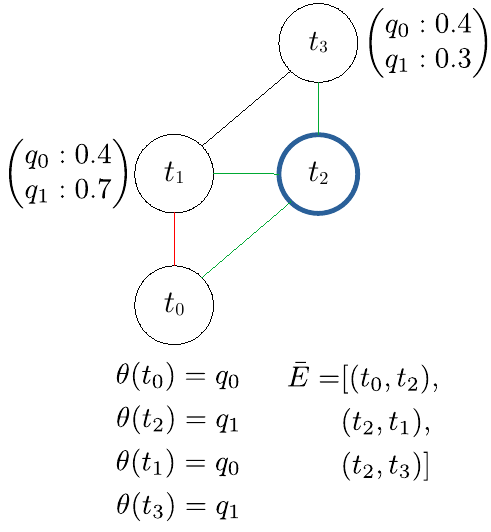}\label{fig:egha_vis3}}%
    \hfil
    \subfloat[]{\includegraphics[width=0.49\columnwidth]{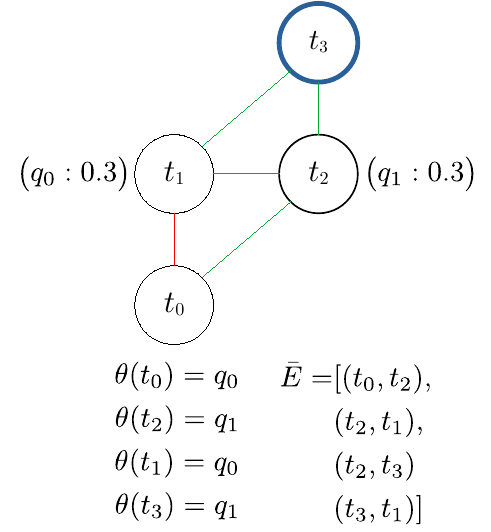}\label{fig:egha_vis4}}%
    \caption{$\epsilon$-GHA propagation process (lines~\ref{alg:egha_begin}-\ref{alg:egha_end} in Alg.~\ref{alg:mod_dijkstra})). Expanded nodes are highlighted with blue, column vectors beside nodes show their IK solutions and corresponding configuration distances to the expanded node, green undirected edges indicate a mapping was found that satisfied the $\epsilon$-GHA condition ($\epsilon = 0.5$ used in this example) and red undirected edges vice versa. \protect\subref{fig:egha_vis1} Starts with root node $t_0$ and arbitrarily assigns mapping $\theta(t_0)$ to its first IK solution $q_0$. \protect\subref{fig:egha_vis2} Only one IK solution $q_1$ for $t_2$ was below $\epsilon$ and thus added to $\theta$ and $\Bar{E}$. \protect\subref{fig:egha_vis3} Shows a case where it's possible for $t_1$ to belong to the same subspace as $t_0$ via an indirect path through $t_2$ and $t_3$. \protect\subref{fig:egha_vis4} When $t_3$ is expanded both neighbours have existing mappings and thus are the only configurations checked, this is important in ensuring existing edges maintain $\epsilon$-GHA condition satisfaction.}
    \label{fig:egha_vis}
\end{figure*}

\begin{algorithm}[t]
	\caption{Decompose task space into $\epsilon$-GHA maps}\label{alg:gen_database}
 \textbf{Input:} robot model $A$, environment \bm{$\hat{m}$} and poses \bm{$\hat{T}$}\\
    \textbf{Output:} maps \bm{$\{\theta^i\}_{i=1}^K$} and corresponding edges \bm{$\{E^i\}_{i=1}^K$}
	\begin{algorithmic}[1]
	\Function{DecomposeTaskSpace}{$\hat{m}$, $\hat{T}$}
	    \State $E \gets \texttt{gen\_edges}(\hat{T})$
	    \State $\hat{T}_{\text{open}} \gets \hat{T}$
	    \State $i \gets 0$
	    \While{$\hat{T}_{\text{open}} \neq \emptyset$}
    	    \State $\hat T_{\text{root}} \gets \texttt{sample}(\hat{T}_{\text{open}})$
    	    \State $J^* \gets \infty$
    	    \For{\textbf{each} $t_0$ in $\hat T_{\text{root}}$}\label{alg:candidate_egha_begin}
                
    	        \State $J(\theta',t_0), \theta', E' \gets$ \Call{GenerateMap}{$G$, $t_0$}
    	        \If{$J(\theta',t_0) < J^*$}
    	            \State $J^*\gets J(\theta',t_0)$
    	            \State $\theta^i, E^i \gets \theta', E'$
    	        \EndIf
    	    \EndFor \label{alg:candidate_egha_end}
    	    \For{\textbf{each} $t$ in $G$ $|$ $\theta^i(t)$ not UNDEFINED}
    	        \State $\hat{T}_{\text{open}} \gets \hat{T}_{\text{open}}  \setminus \{t\}$
    	    \EndFor
    	    \State $i \gets i + 1$
    	\EndWhile
    	\State \Return $\{\theta^i, E^i\}_{i=1}^{K}$
    \EndFunction
	\end{algorithmic}
\end{algorithm}

\begin{algorithm}
	\caption{Search for candidate $\epsilon$-GHA map} \label{alg:mod_dijkstra}
	\hspace*{\algorithmicindent} \textbf{Input:} graph \bm{$G$}, root node \bm{$t_0$}\\
    \hspace*{\algorithmicindent} \textbf{Output:} minimum sum of path costs \bm{$J(\theta', t_0)$} and\\ \hspace*{1cm} corresponding mappings \bm{$\theta'$}
    \vspace{0.1cm}
	\begin{algorithmic}[1]
	\Function{GenerateMap}{$G$, $t_0$}
	    \State $J(\theta', t_0) \gets \infty$
        \For {\textbf{each} $q$ in $Q(t_0)$}
            \State $\Bar{E} \gets \emptyset$ \label{alg:egha_begin}
            \State $\mathbb{Q} \gets t_0$
            \For {\textbf{each} $t$ in $G$}
    		    \State $\theta(t) \gets q$ if $t = t_{0}$, else UNDEFINED
    		    \State $g(\pi(t_0,t_0);\theta)$ $\gets 0$ if $t = t_{0}$, else $c_{\text{max}}$
    	    \EndFor
		    \While {$\mathbb{Q} \neq \emptyset$}
		        \State $t \gets \text{argmin}_{t' \in \mathbb{Q}} g(\pi(t_0,t');\theta)$
		        \State $q_t \gets \theta(t)$
		        \For {\textbf{each } $u$ in $\mathcal{N}_t$} 
		            \If{$\theta(u)$ UNDEFINED} 
    		            \State $q_{u}, l(u,t) \gets \Call{GetMapping}{u, q_t}$
    		        \Else
                        \State $q_{u} \gets \theta(u)$
                        \State $l(u,t)$ $\gets d_C(q_t, \theta(u))$
    		        \EndIf
    		        \State $g(\pi(t_0,u);\theta), \theta \gets \Call{Update}{\mathbb{Q}, q_{u}, u, t}$
		        \EndFor
		    \EndWhile  \label{alg:egha_end}
	        \vspace{0.05cm}
		    \State $J(\theta,t_0) = \sum_{t \in G\setminus t_0}$ $g(\pi(t_0,t);\theta)$
	        \vspace{0.05cm}
		    \If {$J(\theta,t_0) < J(\theta', t_0)$ }
		        \State $J(\theta', t_0) \gets J(\theta,t_0)$
		        \State $\theta', E' \gets \theta, \Bar{E}$ 
		    \EndIf
		\EndFor
	\State \Return $J(\theta', t_0)$, $\theta', E'$
	\EndFunction
	\end{algorithmic}
\end{algorithm}

\begin{algorithm}
	\caption{Get candidate mapping for neighbour node $u$}\label{alg:get_candidate}
	\hspace*{\algorithmicindent} \textbf{Input:} Neighbouring node \bm{$u$} and expanded node mapping \bm{$q_t$}\\
    \hspace*{\algorithmicindent} \textbf{Output:} candidate mapping \bm{$q_{u}$} and resulting edge cost \bm{$l(u,t)$}
	\begin{algorithmic}[1]
        \Function{GetMapping}{$u$, $q_t$}
        \State $l(u,t)$ $\gets  \infty$
        \For {\textbf{each} $p$ in $Q(u)$ $|$ $d_{C}(q_t, p) < \epsilon + d_{T}(u,t)$} {\label{alg:distance_constraint}}
            \If {$d_C(q_t, p) < $ $l(u,t)$}
                \State $q_{u} \gets p$
                \State $l(u,t)$ $\gets d_C(q_t, q_{u})$
            \EndIf
        \EndFor
        \State \Return $q_{u}$, $l(u,t)$
        \EndFunction
	\end{algorithmic}
\end{algorithm}

\begin{algorithm}
	\caption{Update neighbour path cost and node mapping}\label{alg:update_mapping}
	\hspace*{\algorithmicindent} \textbf{Input:} Queue $\mathbb{Q}$, candidate mapping \bm{$q_{u}$}, neighbour node \bm{$u$} and expanded node {$t$} \\
    \hspace*{\algorithmicindent} \textbf{Output:} updated path cost \bm{$g(\pi(t_0, u);\theta)$} and updated IK assignment \bm{$\theta(u)$}
	\begin{algorithmic}[1]
	    \Function{Update}{$\mathbb{Q}$, $q_{u}$, $l(u,t)$}
        \If {$l(u,t)$ $+$ $g(\pi(t_0,t);\theta)$ $<$ $g(\pi(t_0,u);\theta)$}
            \State $\theta(u) \gets q_{u}$
            \State $\Bar{E}.add(u,t)$
    		\State $g(\pi(t_0,u);\theta) \gets l(u,t) + g(\pi(t_0,t);\theta)$
            \State $\mathbb{Q} \gets \mathbb{Q} \cup \{ u \}$
        \EndIf
        \State \Return $g(\pi(t_0, u);\theta)$, $\theta(u)$
    \EndFunction
	\end{algorithmic}
\end{algorithm}

The offline pre-computation stage begins with computation of the $\epsilon$-GHA(s). 
We are given an anticipated environment $\hat{m}$ and a set of tasks $\hat{T}$ representative of online scenarios. We assume the environment is not entirely known in advance, but a general model $\hat{m}$ is available that approximates what is expected. For example, $\hat{m}$ might represent a general bookshelf structure including the shelves and case. However, $\hat{m}$ need not include all objects within the bookshelf, as these details may be unknown \emph{a priori} and discovered later.

An undirected graph $G$ is created with nodes corresponding to poses $t \in \hat{T}$ and edges $\hat{E}$ formed by connecting nodes within a specified radius of each other. An example $\hat{T}$ and $G$ is shown in Figs.~\ref{fig:offline_grid}-\subref*{fig:offline_edges}.
An $\epsilon$-GHA is then generated from $G$ according to Alg.~\ref{alg:gen_database} and is represented as a subgraph of $G$; that is, $(\theta, E)$ where $E \subseteq \hat{E}$. An example subgraph representation is shown in Fig.~\ref{fig:offline_edges_subspace}.

Algorithm~\ref{alg:gen_database} is initialised by assigning all nodes to $\hat{T}_{\text{open}}$.
A node stays in $\hat{T}_{\text{open}}$ until a unique IK solution is assigned. 
While $\hat{T}_{\text{open}}$ is not empty, $\theta$ is found via a \textit{generate map} algorithm procedure~(Alg.~\ref{alg:mod_dijkstra}). Algorithm~\ref{alg:mod_dijkstra} searches for a $\theta$ that minimizes the objective cost, the sum of all \textit{minimum cost} paths $\pi^*(t_0, -)$ from a root node $t_0 \in G$ to all other nodes. That is, 
\begin{equation}\label{eqn:mod_dijk_obj}
J(\theta, t_0) = \sum_{t \in G \setminus t_0} g(\pi^*(t_0, t); \theta),
\end{equation}
where the cost of any path $\pi(t_{0},t_{N-1}) = \{\theta(t_{0}), \ldots,\theta(t_{N-1})\}$ of $N$ nodes is defined as
\begin{equation}
    g(\pi(t_{0},t_{N-1});\theta) = \sum_{n=0}^{N-2} d_C(\theta(t_n),\theta(t_{n+1})).
\end{equation}

To ensure that paths have bounded lengths and do not involve large, unnecessary arm movements, $\theta$ is additionally constrained such that it is an approximate isometry, or an $\epsilon$-Gromov-Hausdorff approximation ($\epsilon$-GHA)~\cite{jaramillo2014structure}, defined below.
\begin{definition}\label{definition_eCHA} The map $\theta : (\hat{T},d_T) \to (\mathcal{C}, d_C)$ is an $\epsilon$-Gromov-Hausdorff approximation if $\forall t_j, t_l \in \hat{T}$
\begin{equation}\label{eGHA}
    |d_T(t_j, t_l) - d_C(\theta(t_j), \theta(t_l))| < \epsilon 
\end{equation} for some $\epsilon > 0$ and metric on the task-space $d_T$.
\end{definition}

It should be noted that $d_T$ and $d_C$ operate on different metric spaces. However, this can be accounted for through appropriate choice of epsilon.

Depending on the value of $\epsilon$, topology of $\hat{m}$ and robot kinematic structure, some nodes may still have undefined mappings after a single iteration of the algorithm. It may then be necessary to search for multiple $\epsilon$-GHAs, $\{ \theta^i \}_{i=1}^K$, and corresponding edges, $\{ E^i \}_{i=1}^K$, that map a covering set of subspaces $\{\hat{T}^i\}_{i=1}^K \subset \hat{T}$ to a set of disjoint configuration subspaces $\{\mathcal{C}^i\}_{i=1}^K \subset \mathcal{C}$.

Note that poses in the example in Fig.~\ref{fig:offline_grid} have the same orientation. For task spaces with differing orientations the algorithm runs identically. As long as a valid task-space distance metric is used,~\eqref{eGHA} holds.

\subsection{$\epsilon$-GHA Computation}\label{sec:egha_comp}

The \textit{generate map} algorithm as outlined in Alg.~\ref{alg:mod_dijkstra} is based on Dijkstra's algorithm and attempts to find a unique IK solution for each task in $\hat{T}$ such that the objective cost in \eqref{eqn:mod_dijk_obj} is minimised. The process is visualised in Fig.~\ref{fig:egha_vis}.
It begins by assigning an undefined mapping to all nodes except for the root node which is mapped arbitrarily to one of its IK solutions. The rest of the procedure is carried out as in the original Dijkstra's algorithm with a priority queue $\mathbb{Q}$~\cite{barbehenn1998note}, with modifications to the node expansion step where we compute the unique IK solution mapping. Here, the set of neighbouring nodes of $t$, $\mathcal{N}_t$, is run through the functions shown in Algs.~\ref{alg:get_candidate}~and~\ref{alg:update_mapping}.

In \textit{get mapping} (Alg.~\ref{alg:get_candidate}), if a node $u \in \mathcal{N}_t$ has already been assigned an IK solution $q_u$, this solution and the edge cost,
\begin{equation}\label{eq:edge_cost}
    l(u,t) = d_C(\theta(u),\theta(t)),
\end{equation}
are returned and the $\epsilon$-GHA $\theta$ is not updated. Otherwise, the IK solution that gives the minimum $l(u,t)$ while satisfying the $\epsilon$-GHA constraint in~\eqref{eGHA} is returned. The returned $q_u$ and $l(u,t)$ are passed as a candidate to \textit{update} (Alg.~\ref{alg:update_mapping}). In \textit{update}, if the candidate path cost from root node to $u \in \mathcal{N}_t$ is less than the current cost then $u$ is mapped to $q_u$, the path cost is updated, and an edge is created between $u$ and $t$. Additionally, if $u$ is not in $\mathbb{Q}$, it is added.

The above process is repeated for all IK solutions of the root node. If required, the algorithm is run $K$ times to find multiple $\epsilon$-GHAs. The resulting one or more $\epsilon$-GHA(s) $\theta^i$ that minimise $J(\theta^i, t_0)$ are returned. All nodes are assigned an IK solution before being placed in the queue. It is also important to note that these assignments do not change during an iteration of the routine, as this could result in unstable behaviour and violation of the $\epsilon$-GHA condition due to changing edge costs. However, path costs may change after finding shorter paths through $G$.

In contrast to the original Dijkstra's algorithm, all nodes are initialised with a non-infinite path cost $c_{\text{max}}$.
Thus, finding a small number of $\epsilon$-GHAs with greater coverage will be favoured, particularly in earlier iterations.

\subsection{Task Sequencing with $\epsilon$-GHAs}\label{sec:task_seq}

\begin{figure*}[!ht]
    \subfloat[]{\includegraphics[trim={0cm 0cm 0.0cm 0cm},clip,width=0.45\columnwidth]{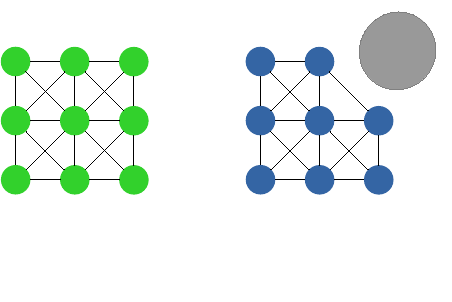}%
    \label{fig:online_plan1}}
    \hfil
    \subfloat[]
    {\includegraphics[width=0.45\columnwidth]{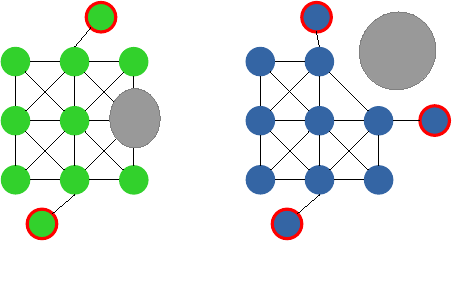}\label{fig:online_plan2}}%
    \hfil
    \subfloat[]
    {\includegraphics[width=0.45\columnwidth]{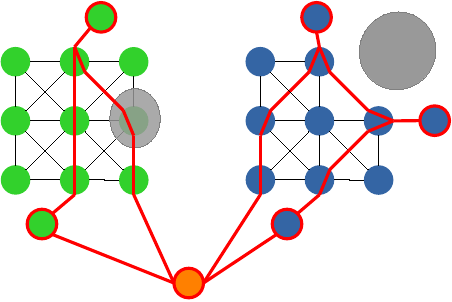}\label{fig:online_plan3}}%
    \hfil
    \subfloat[]{\includegraphics[width=0.45\columnwidth]{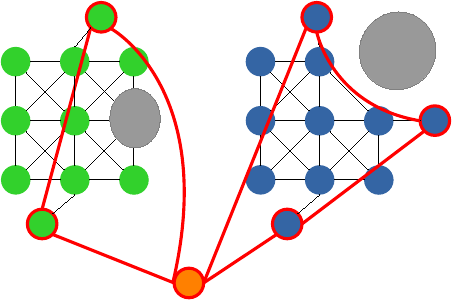}\label{fig:online_plan4}}%
    \caption{Online task sequencing process (Alg.~\ref{alg:online_planner})). Obstacles are grey ellipses, tasks belonging to one subspace are coloured green and the other, blue. \protect\subref{fig:online_plan1} Two subspaces from offline phase and static obstacle. \protect\subref{fig:online_plan2} Online, an obstacle appears and online tasks (highlighted with red border) are matched with closest subspaces and connected to corresponding $\epsilon$-GHA graph. \protect\subref{fig:online_plan3} Tasks grouped by subspace, paths between intra-subspace tasks are computed (only final path sequence shown), home configuration (orange node) appended to each group, sequence computed per group and sequences concatenated. \protect\subref{fig:online_plan4} Paths are adapted and smoothed to account for new obstacles.}
    \label{fig:online_plan}
\end{figure*}

\begin{algorithm}[t]
    \caption{HAP Task Sequencing Overview}\label{alg:online_planner}
	\begin{algorithmic}[1]
	    \State Given an online environment and set of tasks, first compute IK solutions.
            \State Match each task with subspace based on IK solution similarity.\label{alg:online_planner:match}
            \State Group tasks by subspace and retrieve paths between intra-subspace tasks.
            \State Append home configuration to each group to act as intermediate point connecting subspaces.
	    \State Construct weighted adjacency matrix for each group, using retrieved path costs as weights.
	    \State Compute efficient sequence for each group using TSP solver with associated weight matrix and start/end constrained to home configuration.
	    \State Concatenate sequences.\label{alg:online_planner:concat}
	    \State Time-parameterise paths and adapt trajectories to online environment.
	\end{algorithmic}
\end{algorithm}

With each offline task mapped to a unique or reduced set of IK solutions via $\theta$ or $\theta^i$ found offline, the online planner can solve a classical TSP as a proxy for solving the task sequencing problem in Problem~\ref{prob:task_seq}. However, the set of online tasks $T$ provided may not necessarily align with the offline set $\hat{T}$ used to generate the subspace mappings $\theta^i$ in the offline stage, as illustrated in Fig.~\ref{fig:online_scenario}. As such, each task $t \in T$ must first be assigned to one of the subspace mappings $\theta^i$.
This matching process and subsequent task sequencing is outlined in Alg.~\ref{alg:online_planner} and visualised in Fig.~\ref{fig:online_plan}. 

\subsubsection{Matching online tasks to offline subspaces} Beginning with the single map case, to match an online task $t$ with a suitable task-space subspace we query the $k$-closest offline tasks $\{\hat{t}_n\}_{n=1}^k \subset \hat{T}$ to $t$ according to \emph{task-space distance}. For the set of $k$-closest candidates, we retrieve their IK solutions assigned by $\theta$, $\{\theta(\hat{t}_n)\}_{n=1}^k$, and compare them pairwise to all possible IK solutions $q_t \in Q(t)$ for task $t$ using a suitable similarity metric in configuration space. We use the $L_2$-norm in this work. The pair of online/mapped task configurations $(q_t*, \theta(\hat{t}*))$ that minimises this metric is selected as a match.

If multiple $\epsilon$-GHAs were required in the offline stage to decompose the subspace, there may be an overlap, giving a set of suitable matches $\{(q_t^{i}*, \theta^i(\hat{t}^{i}*))\}_{i=1}^{K}$ (one candidate per map) when following the procedure outlined for the single map case above.
To disambiguate this, we choose the first match $q_t^{i}*$ whose similarity value to $\theta^i(\hat{t}^{i}*)$ falls below a given threshold. This reduces the number of checks needed, lowering computation time, and biases matches to subspaces found in earlier iterations, reducing the amount of subspace switching.

Once online tasks are matched, they are connected to the graph of their corresponding $\epsilon$-GHA by adding the edge $(q_t*, \theta(\hat{t}*))$, see Fig.~\ref{fig:online_plan2}. Then online tasks are grouped by equivalent subspace assignment and paths $\pi^*(t_j,t_l)$
are computed pairwise between intra-subspace tasks using a graph search algorithm which can be computed cheaply using the fixed, unique IK solutions in $(\theta^i, E^i)$.

\subsubsection{Task sequencing via classical TSP} With online tasks grouped by subspace and intra-subspace paths generated, a classical TSP can be solved within each subspace. To facilitate moving between subspaces a \emph{home configuration} for the robot is defined. The home configuration is chosen such that all subspaces may be reached via a straight-line path in configuration space. Thus, the home configuration acts as an intermediate point connecting all pairs of subspaces and can be used to connect start and goal poses in separate subspaces by simply joining the straight-line paths to and from it.

To solve the set of classical TSPs in each subspace, weighted adjacency matrices using the path costs $g(\pi^*(t_j,t_l);\theta^i)$ are constructed to define TSP edge costs. Thus, each task group has an associated matrix where the weights correspond to the path cost for each pair of tasks within the group, including the home configuration. These matrices are used as input to a TSP solver, which returns a task sequence for each group independently. The TSP solver is constrained such that each sequence begins and ends at the home configuration. The sequences are concatenated in arbitrary order. This concatenation is always possible by construction of the sequences, which all start and end at the home configuration. 

\subsubsection{Path Post-Processing}\label{sec:traj_adapt}
The final online algorithmic step before a path sequence can be executed by the robot is to time-parameterise the paths and perform post-processing to ensure that they are smooth and time-continuous safe; that is, they are contained entirely within $\mathcal{C}_{\text{free}}$.
We refer to this post-processing as \textit{trajectory adaptation}.

There are a number of existing trajectory optimisation algorithms that are designed for similar purposes. In general, these algorithms require an initial seed trajectory as input, which they attempt to adapt to satisfy given constraints and maximise/minimise a given objective. In this work we can conveniently use the time-parameterised modified subspace path, $\pi^*(t_j,t_l)$, as the seed trajectory.

Unfortunately, trajectory optimisation approaches are not guaranteed to succeed in finding a solution and depend heavily on the choice of seed trajectory. The seed trajectory already accounts for obstacles known at the time of construction and is preferable to less informed choices, such as a straight-line trajectory in configuration space. However, a fallback method remains necessary in case of failure due to obstacles discovered or a task not lying within any subspace at execution time.
In such an event, a probabilistically complete planner is used with the closest IK solution assignment, $q^*_t$ used as its goal configuration input. Probabilistically complete methods may still fail to find a solution in a reasonable amount of time. Thus, HAP enforces a user-defined threshold and terminates execution if exceeded. This is the only condition in which HAP fails to produce an executable trajectory.

\subsection{Practical Considerations}
\subsubsection{Encouraging exploration}\label{sec:exploration}
While searching for $\epsilon$-GHA(s), it is beneficial to bias the search toward the unexplored region of the task space. 
To encourage subspace exploration in subsequent iterations of the routine, a penalty $\rho \cdot \omega(t)$ may be added to all edge costs passing through a node, where $\omega(t)$ is a count of how many times a node is assigned an IK solution, and $\rho$ is a tunable parameter. 
As the penalties $\rho \cdot \omega(t)$ accumulate, previously mapped portions of the graph are not considered in later iterations of the algorithm as their path costs may exceed $c_{\text{max}}$. 
The algorithm then focuses on previously unmapped nodes to increase coverage of $\hat{T}$. 

\subsubsection{Ensuring smooth transitions between subspaces}\label{sec:transition_subspaces}

As previously mentioned, when moving between subspaces the arm first moves back to a fixed home configuration.
Based on a given home pose, the corresponding IK solution that minimises the distance to the average configuration $q_{\text{avg}}^0$ of the first found subspace $\mathcal{C}^0$ are chosen.
Note that the home configuration can be computed online allowing for flexible home pose choice. 

To avoid large changes in configuration while returning to home, the subspaces are biased to be close to one other. For $i \geq 1$ in Alg.~\ref{alg:gen_database} an additional penalty $\rho_s \cdot | \theta^i(u) - q_{\text{avg}}^0 |$ with user-defined weighting $\rho_s$ is added to the edge cost in~\eqref{eq:edge_cost}. This way, chosen IK solutions are biased to be close in configuration space to $q_{\text{avg}}^0$. 
In addition, one can optionally enforce that the IK solution assignment for the root node $t_0$ is within some distance threshold, i.e., $|\theta^i(t_0) - q_{\text{avg}}^0| < \zeta$.

\subsubsection{Balancing the number of subspaces}
If many undefined node mappings remain after an iteration of Alg.~\ref{alg:mod_dijkstra}, $\theta$ may not cover large regions of the task space. This may occur due to
poor flexibility admitted by the robotic manipulator's natural kinematic configurations (demonstrated in Fig.~\ref{fig:reachable_comparison}). This can have adverse impact on online planning if IK solutions of online tasks are far from mapped subspace configurations, leading to failures or jerky  trajectories. As such, Alg.~\ref{alg:gen_database} terminates only when $\hat{T}_{\text{open}} = \emptyset$ and a complete set of subspaces that fully cover the task space are found.
Note that this condition may never be met. For example, a region requiring large configuration changes to switch to may have undefined mapping if a conservative $\zeta$ is chosen.

In practice a large number of subspaces is undesirable as it can slow subspace matching during online planning. Furthermore, frequent subspace switching can add cumbersome overhead to online execution. Thus, to balance a trade-off between coverage and planning/execution time the main loop in Alg.~\ref{alg:gen_database} can be terminated once a user-defined maximum number of subspaces have been found or until the $\zeta$ threshold cannot be satisfied.
Alternatively, the loop can be terminated once a certain task-space coverage percentage has been achieved or when the size of the subspace found in an iteration falls below a set threshold.

\subsubsection{Modifications for mobile bases}\label{sec:mobile_base}

Greater flexibility in subspace assignment and larger workspaces can potentially be achieved by allowing the arm to be mobile. Having multiple $\epsilon$-GHAs synergises well with a mobile manipulator. Instead of allocating an arbitrary number of $\epsilon$-GHAs, one can choose a discrete number of base poses and assign $\epsilon$-GHAs to each base pose.

When computing the task-space decomposition for a mobile base, the \textit{generate map} algorithm in Alg.~\ref{alg:mod_dijkstra} is run for all base poses in each iteration in Alg.~\ref{alg:gen_database}. The pose that yields the lowest objective cost is allocated a $\epsilon$-GHA.
This base pose is then removed as a candidate in the subsequent iterations. This is repeated until all base poses have been assigned a subspace or until $\hat{T}_{\text{open}} = \emptyset$.

For mobile base online execution, the subspace switching action consists of moving back to the home configuration before moving the base. Furthermore, when sequencing for the mobile base case we ignore the base movement cost when switching between subspaces. However, if this were important then sequencing the group execution order could be solved using another TSP with the base movement costs.

\subsubsection{Task-space graph construction} \label{sec:practical:graph}
An example of an undirected graph on $\hat{T}$ is visualised in Fig.~\ref{fig:offline_edges}, where nodes are tasks in $\hat{T}$ and edges are connections between tasks. The construction strategy of such a graph is flexible, however there are a few important considerations to highlight.
Firstly, greater edge connection density allows for more diverse paths and greater node density increases the probability of time continuous safety being met by the retrieved subspace paths, $\pi^*(t_j,t_l)$.
Additionally, edges connect nodes in task space and hence a local connection strategy needs to be devised to ensure they are feasible in the configuration space. The connection strategy used in this paper is to connect nodes within a ball of specified radius in the workspace. To ensure feasibility, it is required that an IK solution in $\mathcal{C}_{\text{free}}$ exists for a discrete set of points along the edge connecting two nodes.

The $\hat{T}$ upon which the graph is built is a discretisation of the user-defined task space. In the context of this work, it represents approximate poses that the arm is expected to plan to. An example discretisation strategy is visualised in Fig.~\ref{fig:offline_grid}. Poses here could represent abstract tasks such as pre-grasp points or camera viewpoints for active perception. The orientation of the poses in $\hat{T}$ should align roughly with the expected tasks. For example, in Fig.~\ref{fig:offline_grid} all poses point forward into the bookshelf, a suitable construction for tasks such as grasping and scene reconstruction. Notice in the online scenario in Fig.~\ref{fig:online_scenario} the task poses need not lie exactly on $\hat{T}$. However, performance may decrease the greater this discrepancy is. While this is a suitable choice for the given environment our framework is perfectly capable of searching over task spaces with varying orientations as long as a valid task space distance metric is used, see sec V.

The poses in this work are generated procedurally by defining a uniform graph of nodes across a volume bounding the bookshelf; however, there exists many possible methods for generating these poses. For example, the manipulator could be teleoperated to various poses or the manipulator could be moved kinesthetically. That is, the human operator could move the end effector directly and store the poses. However, care should be taken such that no islands are formed in the graph. For example, in the bookshelf scenario there is a plane of poses in front of the bookshelf to ensure that there is connectivity between poses within the shelves. However, this could potentially be resolved in a post-processing step automatically, relieving the burden on the operator.

\section{ANALYSIS}
In this section, we show that the paths found by Alg.~\ref{alg:gen_database} have bounded lengths $d_{C}$ and are hence efficient and free of jerky motion. 
This is because, as we show, $\epsilon$-GHAs approximately preserve shortest paths between metric spaces, in our case the task and configuration spaces. 
We first verify that the map $\theta$ found by Alg.~\ref{alg:gen_database} is indeed an $\epsilon$-GHA (definition.~\ref{definition_eCHA}).
\begin{theorem}[$\theta$ is an $\epsilon$-GHA]
$\theta$ found by Alg.~\ref{alg:mod_dijkstra} is an $\epsilon$-GHA. 
\end{theorem}

\begin{proof}
Pick any $t\in\hat{T}$. Then $\forall u_1 \in \mathcal{N}_{t}$, by construction, we have 
\begin{equation*}
    |d_C(\theta(t), \theta(u_1)) - d_T(t, u_1)| < \epsilon,
\end{equation*}
for some $\epsilon > 0$.
Then, $\forall u_2 \in \mathcal{N}_{u_1}$, i.e. the next-nearest neighbours of $t$, we again have by construction
\begin{equation*}
    |d_C(\theta(u_1), \theta(u_2)) - d_T(u_1, u_2)| < \epsilon.
\end{equation*}
Similarly for all $(N-1)$th and $N$th nearest neighbours of $t$,
\begin{equation*}
    |d_C(\theta(u_{N-1}), \theta(u_N)) - d_T(u_{N-1}, u_N)| < \epsilon.
\end{equation*}
Then, using the triangle inequality we get
\begin{equation*}
\begin{array}{l}
    |d_C(\theta(t), \theta(u_N)) - d_T(t, u_N)| \\
    \leq |d_C(\theta(t), \theta(u_1)) - d_T(t, u_1)| +\\ 
    \sum_{n=1}^{N-1}|d_C(\theta(u_n), \theta(u_{n+1})) - d_T(u_n, u_{n+1})|\\
    < \epsilon + (N-1)\epsilon\\
    = N\epsilon.
\end{array}
\end{equation*}
As $\epsilon$ is arbitrary, taking $\epsilon$ to be $N\epsilon$ gives the required result.
\end{proof}

Furthermore, $\theta$ maps shortest paths in the workspace to paths in configuration space that are of bounded length. That is, \emph{minimising geodesics} are approximately preserved under the mapping. Intuitively, minimising geodesics (referred to henceforth as geodesics for brevity) are a generalisation of ``straight lines", or shortest paths, in Euclidean space to more general spaces, defined below.

A metric space $(X, d_{X})$ is a set $X$ equipped with a metric, or ``distance function", $d_{X} : X \times X \rightarrow \mathbb{R}$ that satisfies the axioms of positiveness, symmetry, and triangular inequality~\cite{course_metric_geometry}.
Drawing upon concepts from metric geometry, we characterise geodesics on metric spaces as paths whose segment lengths sum to that of the whole path length. Formally,

\begin{definition}[Geodesics]\label{def:geodesic}
Given a metric space $(X, d_{X})$ with intrinsic metric $d_{X}$, a path $\gamma : [0, 1] \rightarrow X$ is a geodesic iff:
\begin{equation}\label{eq:geodesic}
    d_{X}(\gamma(0), \gamma(1)) = \sum_{n} d(\gamma(s_{n}), \gamma(s_{n+1})),
\end{equation}
for any ordered $\{s_{n}\} \subset [0, 1]$ with first and last entries equal to $0$ and $1$, respectively. 
\end{definition}%

We now show that an $\epsilon$-GHA, and thus $\theta$ found by HAP, preserves geodesics approximately.

\begin{theorem}[$\epsilon$-GHAs preserve geodesics]
Let $(X, d_{X})$ and $(Y, d_{Y})$ be two metric spaces, and $\theta : X \rightarrow Y$ an $\epsilon$-GHA. Then, for any geodesic $\gamma$ on $X$, we have
\begin{equation}
\begin{split}
    & |d_{Y}(\theta(\gamma(0)), \theta(\gamma(1))) - 
     \sum_{n} d_{Y}(\theta(\gamma(s_{n})), \theta(\gamma(s_{n+1})))| \\
    & \leq (N+1) \epsilon.
\end{split}
\end{equation}
In other words, the path $\theta(\gamma(s)) : [0,1] \to Y$ is $(N+1) \epsilon$ away from being a geodesic in the configuration space.  
\end{theorem}

\begin{proof}
Applying the $\epsilon$-GHA condition to the end-points of $\gamma$, we have
\begin{equation*}\label{eq:proof-egha}
    | d_{Y}(g(\gamma(0)), g(\gamma(1))) - d_{X}(\gamma(0), \gamma(1)) | \leq \epsilon. 
\end{equation*}

Because $\gamma$ is a geodesic, we can replace the $d_{X}(\gamma(0), \gamma(1))$ term above with the sum of lengths along any test points $\{s_{1}, ..., s_{N}\}$,
\begin{equation*}
    | d_{Y}(g(\gamma(0)), g(\gamma(1))) - \sum_{n} d_{X}(\gamma(s_{n}), \gamma(s_{n+1})) | \leq \epsilon.
\end{equation*}

Using the $\epsilon$-GHA condition on summands individually, we have
\begin{equation*}
\begin{split}
     & \sum_{n} d_{Y}(g(\gamma(s_{n})), g(\gamma(s_{n+1}))) - (N+1)\epsilon \\
     & \leq d_{Y}(g(\gamma(0)), g(\gamma(1))) \\
     &\leq \sum_{n} d_{Y}(g(\gamma(s_{n})), g(\gamma(s_{n+1}))) + (N+1)\epsilon.    
\end{split}
\end{equation*}
$N+1$ arises because there are $N$ epsilons inside the sum and one outside. 
\end{proof}

\section{EXPERIMENTS}

\begin{figure*}[t!]
\centering
\subfloat[]{\includegraphics[trim={0.6cm 0.0cm 1.3cm 0.0cm},clip,width=0.329\textwidth]{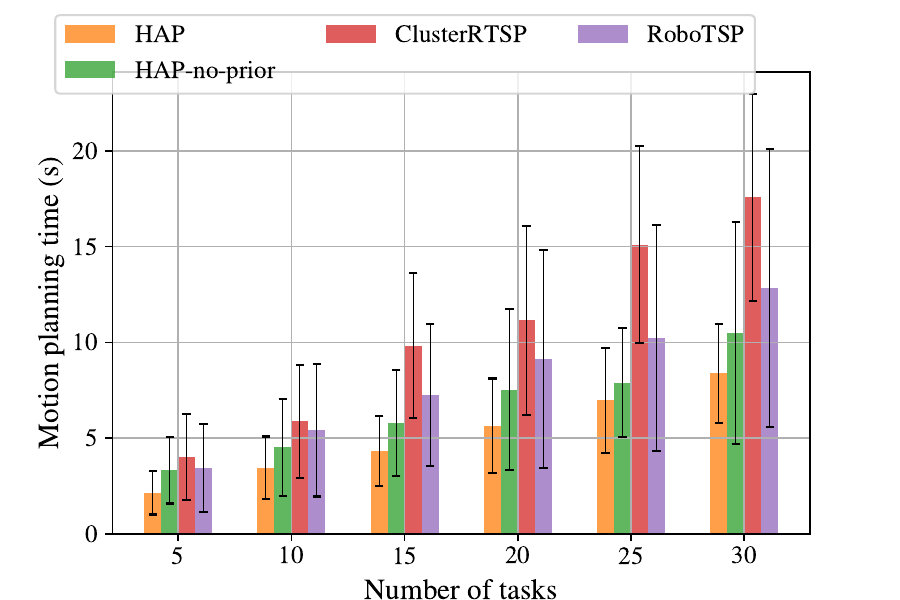}%
\label{fig:sawyer_motion_plan}}
\hfil
\subfloat[]{\includegraphics[trim={0.0cm 0.0cm 1.4cm 0.0cm},clip,width=0.345\textwidth]{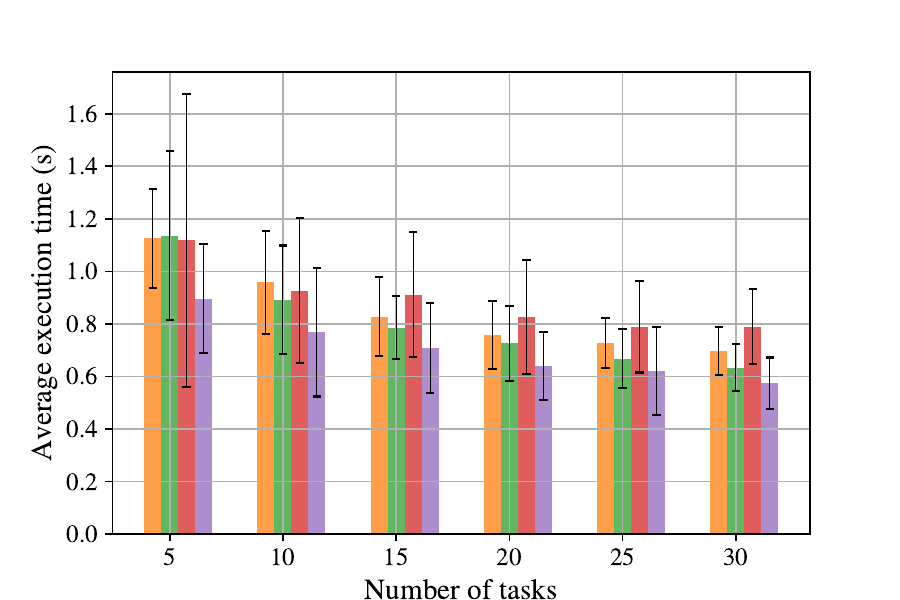}%
\label{fig:sawyer_exec_time}}
\hfil
\subfloat[]{\includegraphics[trim={0.6cm 0.0cm 1.3cm 0.0cm},clip,width=0.324\textwidth]{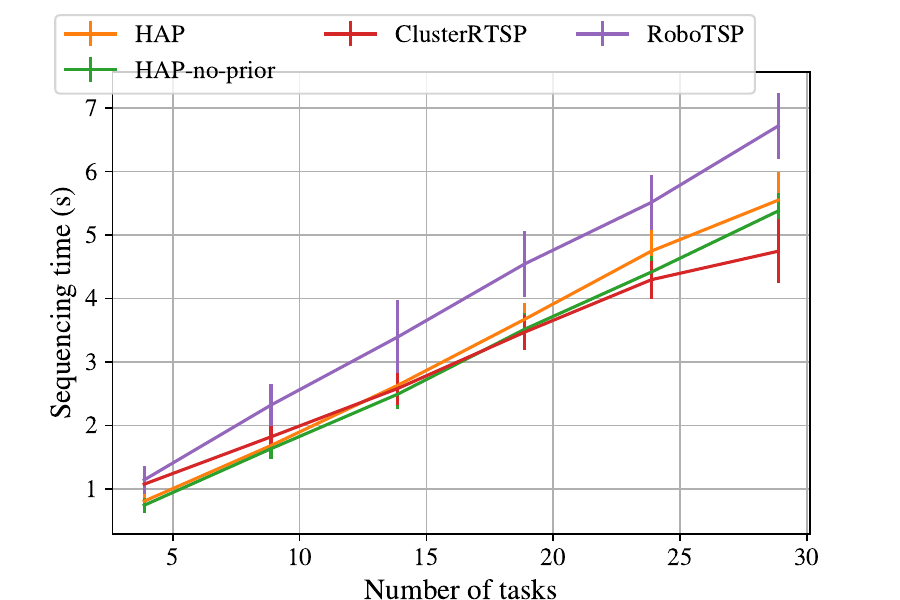}%
\label{fig:sawyer_sequence_time}}
\hfil
\caption{Time benchmarks for Sawyer bookshelf experiments with varying number of tasks. \protect\subref{fig:sawyer_motion_plan}~Total motion planning times for HAP variants remain low while RoboTSP and Cluster-RTSP steadily increases. \protect\subref{fig:sawyer_exec_time}~Average trajectory execution times for all methods decrease with increasing number of tasks. \protect\subref{fig:sawyer_sequence_time}~Task sequencing times increase relatively linearly across all benchmarks with RoboTSP being slightly worse.}
\label{fig:sawyer_time_benchmarks}
\end{figure*}

\begin{figure*}[t!]
\centering
\subfloat[]{\includegraphics[trim={0.6cm 0.0cm 1.5cm 0.0cm},clip,width=0.329\textwidth]{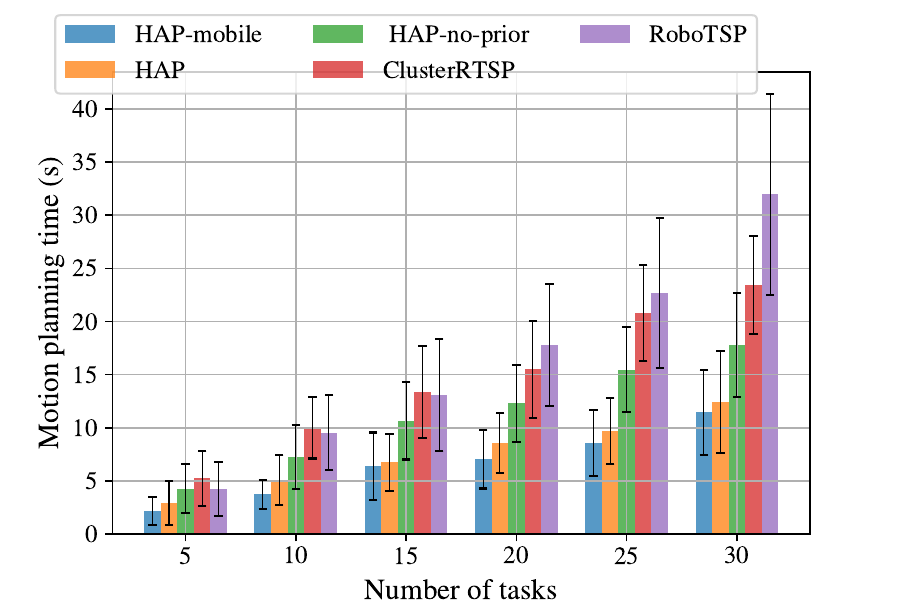}%
\label{fig:ur5_motion_plan}}
\hfil
\subfloat[]{\includegraphics[trim={0.0cm 0.0cm 1.4cm 0.0cm},clip,width=0.345\textwidth]{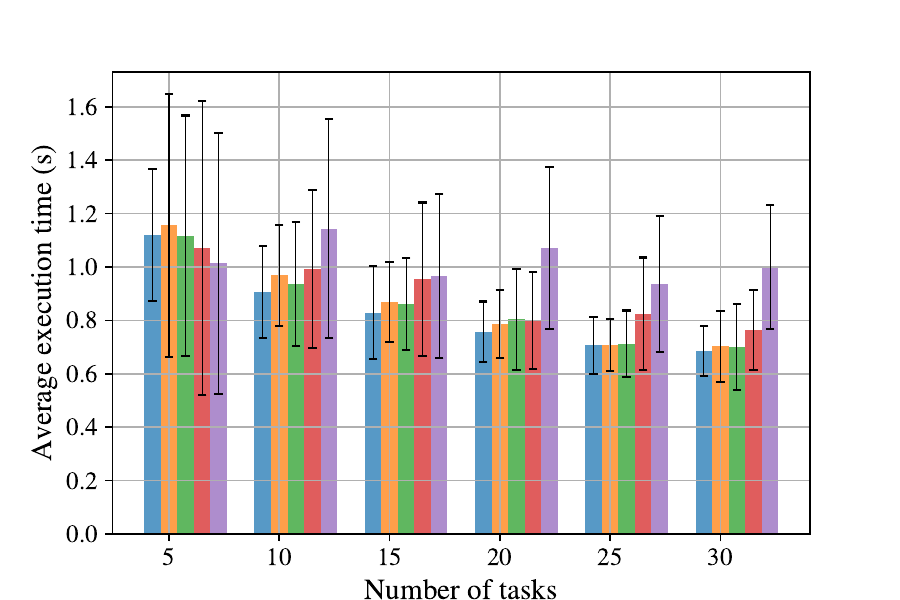}%
\label{fig:ur5_exec_time}}
\hfil
\subfloat[]{\includegraphics[trim={0.3cm 0.0cm 1.3cm 0.0cm},clip,width=0.324\textwidth]{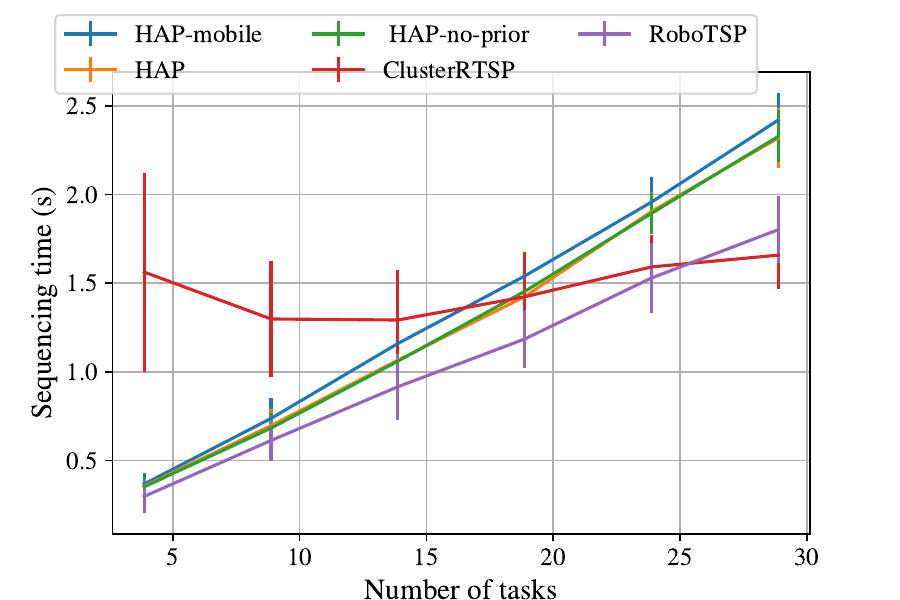}%
\label{fig:ur5_sequence_time}}
\hfil
\caption{Time benchmarks for UR5 bookshelf experiments with varying number of tasks. \protect\subref{fig:ur5_motion_plan}~Total motion planning times show HAP benefiting more from subspace trajectory priors and RoboTSP's performance deteriorating. \protect\subref{fig:ur5_exec_time}~Average trajectory execution times for HAP and Cluster-RTSP decrease while RoboTSP remains
approximately constant. \protect\subref{fig:ur5_sequence_time}~Task sequencing times follow a similar trend to Fig.~\ref{fig:sawyer_sequence_time} with the exception of Cluster-RTSP.}
\label{fig:ur5_time_benchmarks}
\end{figure*}

\begin{figure}[t!]
\centering
\subfloat[Sawyer bookshelf experiments]{\includegraphics[width=0.4\textwidth]{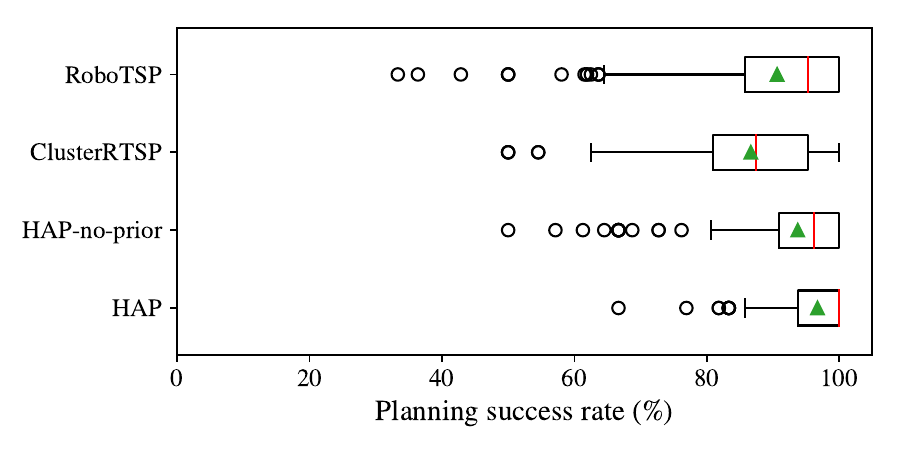}%
\label{fig:sawyer_plan_success}}
\hfil
\subfloat[UR5 bookshelf experiments]{\includegraphics[width=0.4\textwidth]{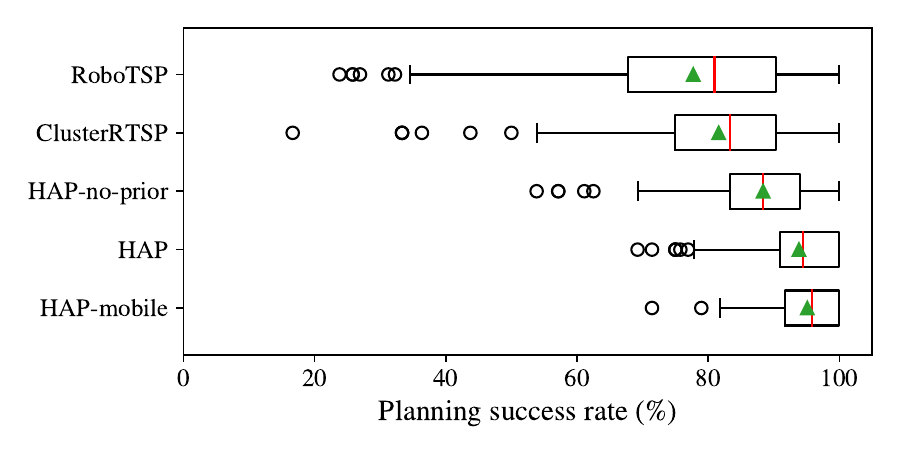}%
\label{fig:ur5_plan_success}}
\hfil
\caption{Planning success rates for \protect\subref{fig:sawyer_plan_success}~Sawyer and \protect\subref{fig:ur5_plan_success}~UR5 across all bookshelf experiments. Red lines indicate medians and green triangles means.}
\label{fig:bookshelf_plan_success}
\end{figure}

This section presents several sets of experimental results that demonstrate and evaluate HAP's performance in a complex simulated bookshelf environment using robot manipulators with varying kinematic models. For comparison, we also present results obtained using state-of-the-art task sequencing methods.

\subsection{Experimental Setup}
Experiments consist of sets of randomised trials where tasks are sampled uniformly at random from the environment's task space. The number of tasks range from 5 to 30 and for each setting we run 50 trials to gauge the robustness of each method.

We compare HAP with competitive contemporary RTSP solver baselines. Both methods compute a sequence of goal IK solutions and then use the same trajectory adaption process as HAP, except with a straight-line trajectory prior for the trajectory optimisation planner:
\begin{itemize}
    \item \textit{RoboTSP~\cite{suarez2018robotsp}:} Initially solves a TSP in task space. Given this sequence, an optimal joint configuration is assigned to each task using a graph search algorithm such that the total path length (Euclidean distance in configuration space) through each configuration is minimised.
    \item \textit{Cluster-RTSP~\cite{wong2020novel}:} Begins by assigning a unique configuration to each task based on a best-fit heuristic. This essentially tries to choose a set of configurations that are all close to each other. The configurations are then clustered into similar groups, based on proximity in configuration space. A configuration sequence is then found by solving inter-cluster and intra-cluster TSPs individually.
\end{itemize}

The following metrics are evaluated:
\begin{itemize}
    \item \textit{Planning success rate} refers to the percentage of tasks for which the trajectory adaption process succeeded in motion planning a trajectory in $\mathcal{C}_{\text{free}}$. 
    \item \textit{Motion planning time} refers to the total computation time taken for the trajectory adaption process to compute motion plans for all tasks. 
    \item \textit{Average execution time} refers to the average time taken per task for a successful trajectory execution, assuming the arms are operating at their maximum joint velocities. Note that for HAP we count inter-subspace trajectories, that is the transition to and from the home configuration, as one task.
    \item \textit{Maximum trajectory jerk} refers to the maximum joint jerk norm of an executed trajectory. This is computed numerically using the finite difference method.
    \item \textit{Sequencing time} for HAP refers to the time taken to match tasks, retrieve subspace paths and then compute a sequence (Alg.~\ref{alg:online_planner}~lines~\ref{alg:online_planner:match}-\ref{alg:online_planner:concat}). Additionally, online task IK solution computation time is included for all methods.
\end{itemize}

The robot models used are a 7-DOF Rethink Robotics Sawyer and a 6-DOF Universal Robots UR5. The UR5 also has the option to be mobile, which we refer to as HAP-Mobile. 
We additionally include a variant of HAP which, similar to RoboTSP and Cluster-RTSP, utilises a straight-line trajectory prior for trajectory optimisation, referred to as HAP-no-prior.

All benchmarks are run on the bookshelf environment shown in Fig.~\ref{fig:hap_overview}. The task-space decomposition is carried out with the empty bookshelf and in the online scenarios objects are added. The computed subspaces are shown in Figs.~\ref{fig:reachable_comparison}-\ref{fig:ur5_mobile_subspaces}.

\subsection{Implementation Details}

\begin{table}[]
\begin{minipage}{\columnwidth}
    \centering
    \begin{center}
    \begin{tabular}{ ll } 
    \toprule
    \textbf{HAP} & $\bm{501.00 \pm 203.56}$ \\
    \midrule
    HAP-no-prior & $1415.43 \pm 820.88$ \\
    \midrule
    Cluster-RTSP~\cite{wong2020novel} & $2607.88 \pm 2130.33$ \\
    \midrule
    RoboTSP~\cite{suarez2018robotsp} & $2390.36 \pm 1269.95$ \\
    \bottomrule
    \end{tabular}
    \end{center}
    \caption{Maximum trajectory jerk values (rad $\cdot$ s$^{-3}$) for Sawyer bookshelf experiments. Mean and standard deviation are computed over all successful trajectories.}
    \label{table:sawyer_jerks}
\end{minipage}\hfill
\begin{minipage}{\columnwidth}
    \centering
    \begin{center}
    \begin{tabular}{ ll }%
    \toprule
    HAP-mobile & $1968.07 \pm 772.99$ \\
    \midrule
    \textbf{HAP} & $\bm{1743.73 \pm 649.29}$ \\
    \midrule
    HAP-no-prior & $4048.48 \pm 2378.24$ \\
    \midrule
    CluserRTSP~\cite{wong2020novel} & $5651.06 \pm 4776.34$ \\
    \midrule
    RoboTSP~\cite{suarez2018robotsp} & $3615.76 \pm 2664.09$ \\
    \bottomrule
    \end{tabular}
    \end{center}
    \caption{Maximum trajectory jerk values (rad $\cdot$ s$^{-3}$) for UR5 bookshelf experiments.}
    \label{table:ur5_jerks}
\end{minipage}
\end{table}

In implementing Alg.~\ref{alg:gen_database}, the number of root nodes, $\| \hat{T}_{\text{root}} \| = 10$. The maximum number of $\epsilon$-GHAs was set to $5$ for the UR5 experiments. Sawyer experiments only utilised a single $\epsilon$-GHA~due the flexibility in IK solutions afforded by its kinematic redundancy.
In implementing Algs.~\ref{alg:mod_dijkstra}-\ref{alg:update_mapping}, all nodes are initialised with path cost $c_{\text{max}} = 5.0$. The subspace exploration penalty $\rho=2.0$. The subspace distance biasing penalty $\rho_{s}=0.02$. Experiments with UR5 model use $\epsilon=0.35$, and those with Sawyer model use $\epsilon=0.85$. Distance $d_C$ is defined as $L_{\infty}$, and $d_T$ as $L_2$. Both metrics use uniformly weighted distances.

For the online planner (Alg.~\ref{alg:online_planner}), the number of closest neighbours used when retrieving $\pi^*(t_j,t_l)$ is $k=10$. 
The threshold used for terminating the search over mappings when matching a task configuration is $L_2$ distance $0.7$. Google's or-tools~\cite{ortools} package is used as the TSP-solver in~\eqref{eq:task_sequencing}. 

The discretised task space, $\hat{T}$, used is shown in Fig.~\ref{fig:offline_grid}.
The home pose used is shown in Fig.~\ref{fig:online_scenario}.
For the mobile base experiments, a discrete set of possible base positions is defined uniformly along the $y$-axis, parallel to the bookshelf, see Fig.~\ref{fig:ur5_mobile_subspaces}.

The trajectory optimisation algorithm we use for online adaption is TrajOpt with default implementation from~\cite{ortrajoptGit}. The fallback probabilistic method used is BIT* with a timeout limit of 2 seconds and otherwise default implementation from~\cite{omplGit}.

The simulation environment used is OpenRAVE with IKFast kinematics solver~\cite{diankov2008openrave}. To generate a finite set of IK solutions for the 7-DOF Sawyer, IKFast sets the second DOF as a free joint with discretisation increments of $0.01$ radians.

For the benchmark methods RobotTSP and Cluster-RTSP, the default parameters are used in the provided code implementations~\cite{robotspGit,clusterrtspGit}. The only difference being that the maximum number of clusters for Cluster-RTSP was modified to be the number of tasks.

\subsection{Results}

Time benchmarks for the Sawyer and UR5 bookshelf experiments are shown in Figs.~\ref{fig:sawyer_time_benchmarks}~and~\ref{fig:ur5_time_benchmarks}. Average maximum trajectory jerk values are shown in Tables~\ref{table:sawyer_jerks}~and~\ref{table:ur5_jerks}.

The box-and-whisker plots in Fig.~\ref{fig:bookshelf_plan_success} show that HAP variants achieved favourable success rate relative to baseline methods. Notably, the UR5 benefits from subspace trajectory priors more so than the Sawyer which is able to utilise its kinematic redundancy to better navigate around obstacles. 

Motion planning times for HAP variants in the UR5 experiments greatly outperform benchmarks with approximately up to 3x speedup compared to robotsp and 2x compared to cluster-RTSP with consistently higher planning success rates and whilst retaining similar task sequencing time scaling. For a plan to fail, both TrajOpt and BIT* must timeout. Thus, low planning success rates explain the higher motion planning times exhibited by the baseline methods.

HAP variants and Cluster-RTSP saw a downward trend in execution times as the number of tasks increased for both experiments, whereas RoboTSP's fluctuated for the UR5 case. Whilst HAP's execution times were similar to Cluster-RTSP, and RoboTSP in the case of Sawyer experiments, the average maximum trajectory jerk values in Tables~\ref{table:sawyer_jerks}~and~\ref{table:ur5_jerks} were up to approximately 5x lower with low variance. It should be noted that trajectory execution times describe successful trials only and should be interpreted in conjunction with corresponding planning success rates. 

Furthermore, the HAP-mobile variant achieved comparable results to the static case, confirming HAP's ability to generalise to mobile bases. Interestingly, the mobile and static variants utilised on average 2.45 and 1.9 subspaces per trial, respectively. This can be explained by the greater subspace coverage and diversity provided by the mobile base (see Fig.~\ref{fig:ur5_subspaces}), hence the slightly better planning and success rates.

Total offline-computation time was 108s, 318s and 569s for the UR5, Sawyer and UR5 with mobile base, respectively. It should be noted, however, that this could be substantially sped up through various optimisations such as parallelisation of the candidate $\epsilon$-GHA map search (Alg.~\ref{alg:gen_database}~lines~\ref{alg:candidate_egha_begin}-\ref{alg:candidate_egha_end}) and IK solution computation.

\subsection{Discussion}

\begin{figure}[!ht]
\centering
\subfloat[$\sigma_{RoboTSP}\lbrack2:3\rbrack $]{\includegraphics[width=0.49\columnwidth]{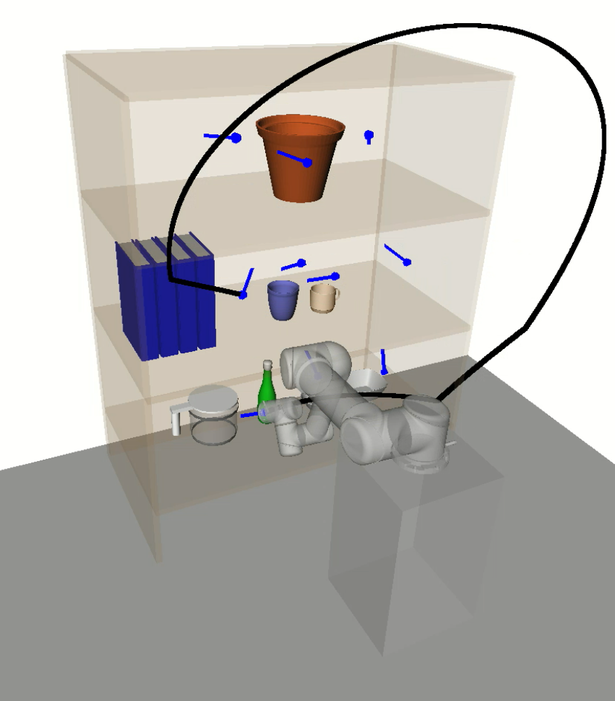}
\label{fig:ur5_robotsp_seq3}}
\hfil
\subfloat[$\sigma_{RoboTSP}\lbrack5:6\rbrack $]{\includegraphics[width=0.49\columnwidth]{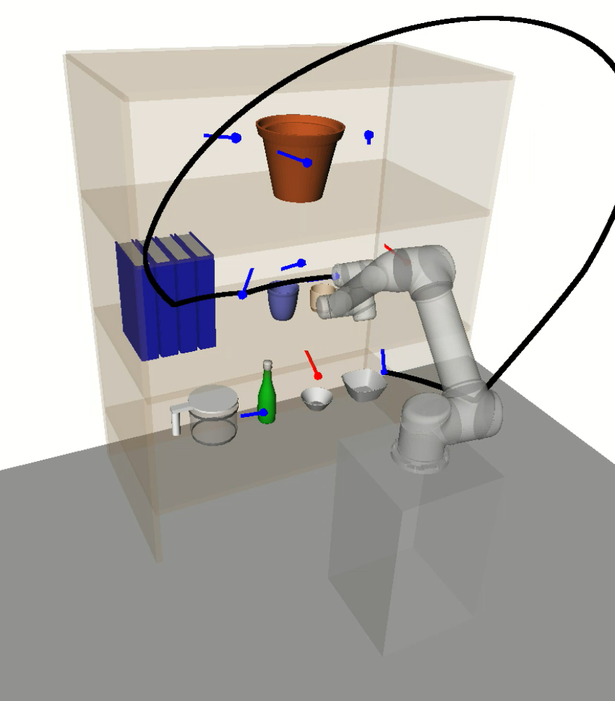}%
\label{fig:ur5_robotsp_seq5}}
\hfil
\subfloat[$\sigma_{RoboTSP}\lbrack7:9\rbrack $]{\includegraphics[width=0.49\columnwidth]{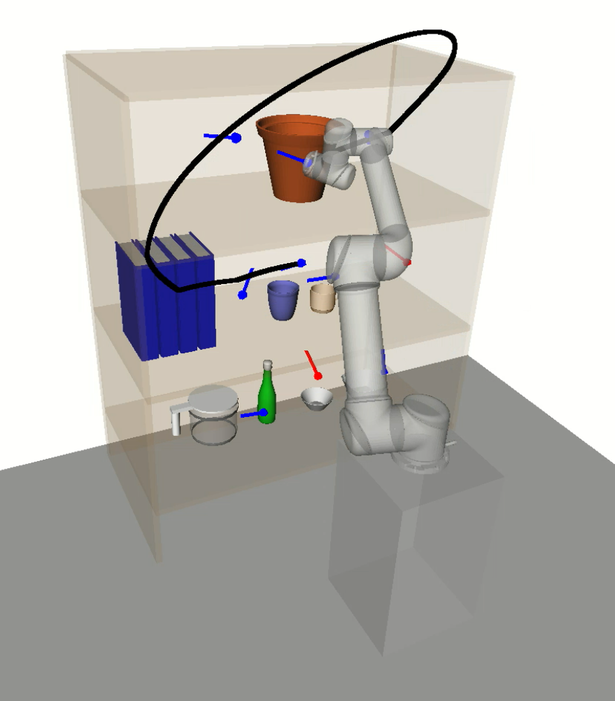}%
\label{fig:ur5_robotsp_seq7}}
\hfil
\subfloat[$\sigma_{RoboTSP}\lbrack9:10\rbrack $]{\includegraphics[width=0.49\columnwidth]{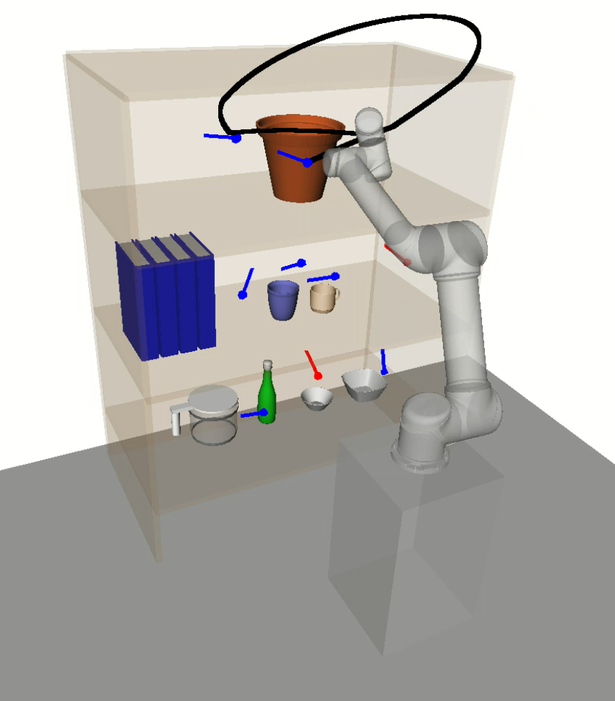}%
\label{fig:ur5_robotsp_seq8}}
\hfil
\caption{Example subset of a RoboTSP trajectory sequence for UR5 bookshelf experiment. The trajectories for the tasks shown highlight the issue of decoupling task and configuration space during sequencing; short paths in tasks space are not necessarily so in configuration space. Red poses indicated failed plans.}
\label{fig:ur5_robotsp_seq}
\end{figure}

\begin{figure}[!ht]
\centering
\subfloat[$\sigma_{ClusterRTSP}\lbrack6:7\rbrack $]{\includegraphics[width=0.49\columnwidth]{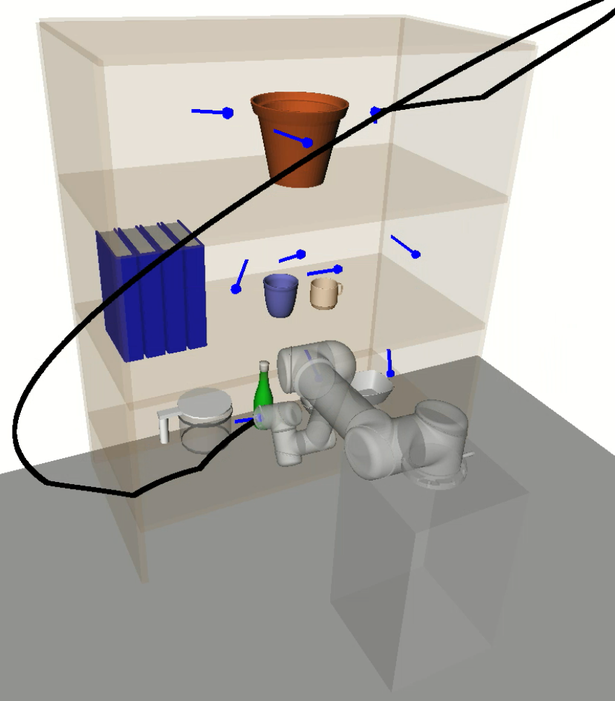}
\label{fig:ur5_cluster_seq7}}
\hfil
\subfloat[$\sigma_{ClusterRTSP}\lbrack8:9\rbrack $]{\includegraphics[width=0.49\columnwidth]{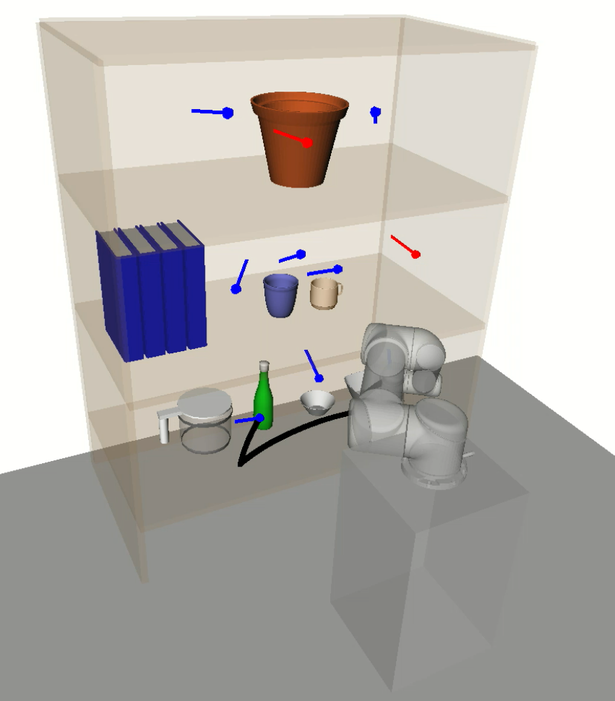}%
\label{fig:ur5_cluster_seq9}}
\hfil
\caption{Example subset of a Cluster-RTSP trajectory sequence for UR5 bookshelf experiment. A large change in joint configuration occurs mid sequence and motion planning fails to two subsequent tasks (red poses).}
\label{fig:ur5_clustertsp_seq}
\end{figure}

\begin{figure}[!h]
\centering
\subfloat[$\sigma_{HAP}\lbrack0:1\rbrack, \theta^0 $]{\includegraphics[width=0.15\textwidth]{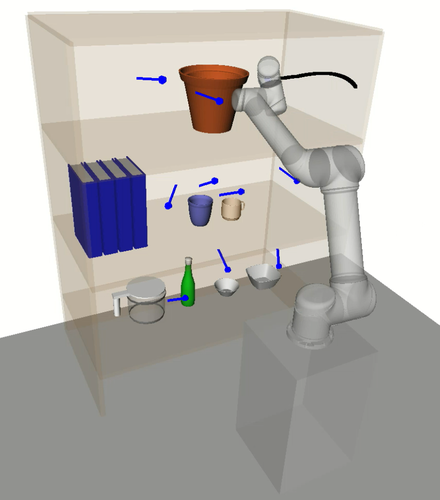}%
\label{fig:ur5_hap_seq1}}
\hfil
\subfloat[$\sigma_{HAP}\lbrack1:2\rbrack, \theta^0 $]{\includegraphics[width=0.15\textwidth]{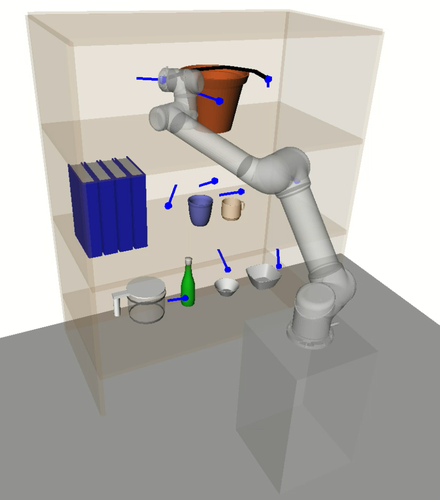}%
\label{fig:ur5_hap_seq2}}
\hfil
\subfloat[$\sigma_{HAP}\lbrack2:3\rbrack, \theta^0 $]{\includegraphics[width=0.15\textwidth]{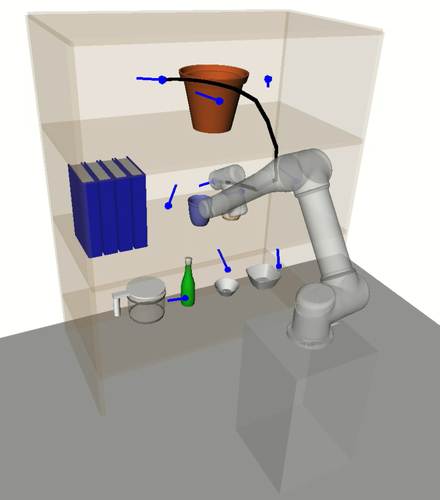}%
\label{fig:ur5_hap_seq3}}
\hfil
\subfloat[$\sigma_{HAP}\lbrack3:4\rbrack, \theta^0 $]{\includegraphics[width=0.15\textwidth]{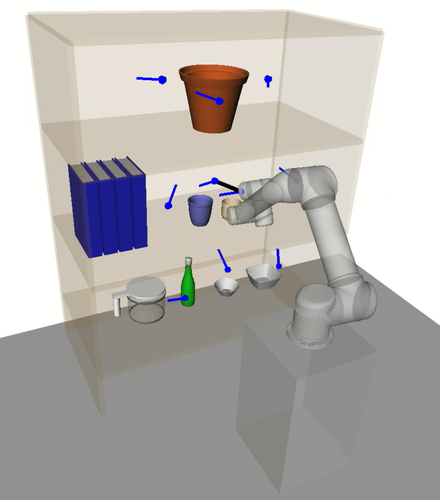}%
\label{fig:ur5_hap_seq4}}
\hfil
\subfloat[$\sigma_{HAP}\lbrack4:5\rbrack, \theta^0 $]{\includegraphics[width=0.15\textwidth]{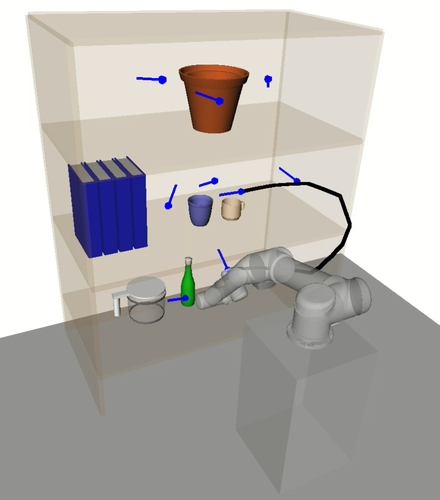}%
\label{fig:ur5_hap_seq5}}
\subfloat[$\sigma_{HAP}\lbrack5:0\rbrack, \theta^1 $]{\includegraphics[width=0.15\textwidth]{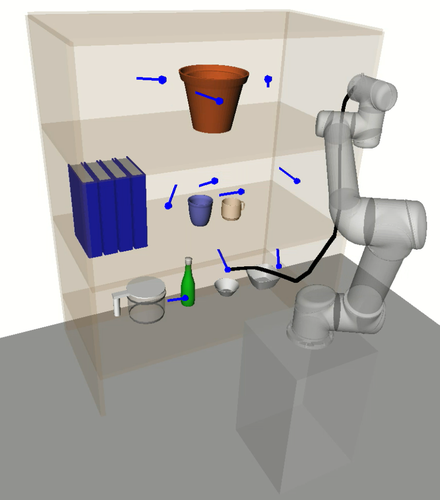}%
\label{fig:ur5_hap_seq6}}
\hfil
\subfloat[$\sigma_{HAP}\lbrack0:6\rbrack, \theta^1 $]{\includegraphics[width=0.15\textwidth]{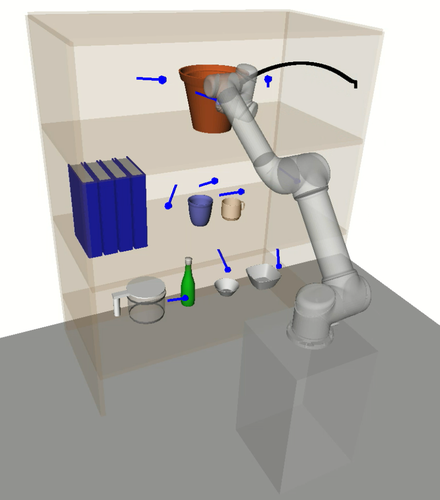}%
\label{fig:ur5_hap_seq7}}
\hfil
\subfloat[$\sigma_{HAP}\lbrack6:7\rbrack, \theta^1 $]{\includegraphics[width=0.15\textwidth]{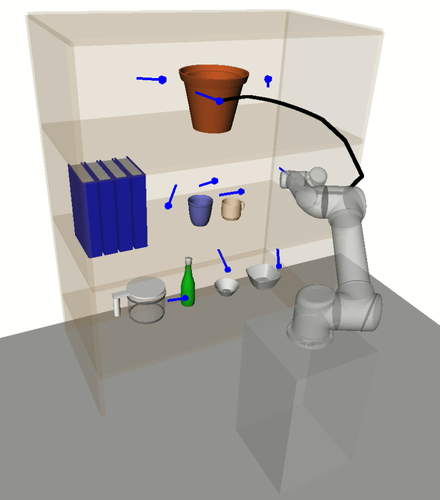}%
\label{fig:ur5_hap_seq8}}
\hfil
\subfloat[$\sigma_{HAP}\lbrack7:8\rbrack, \theta^1 $]{\includegraphics[width=0.15\textwidth]{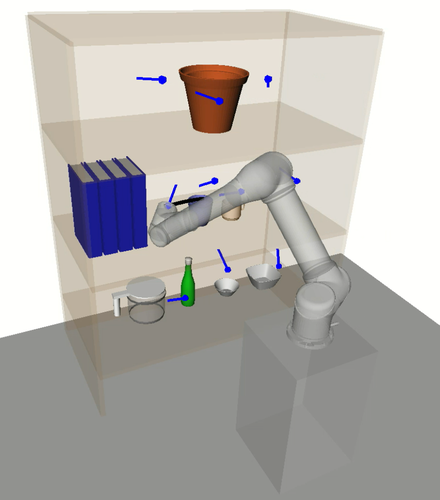}%
\label{fig:ur5_hap_seq9}}
\hfil
\subfloat[$\sigma_{HAP}\lbrack8:0\rbrack, \theta^1 $]{\includegraphics[width=0.15\textwidth]{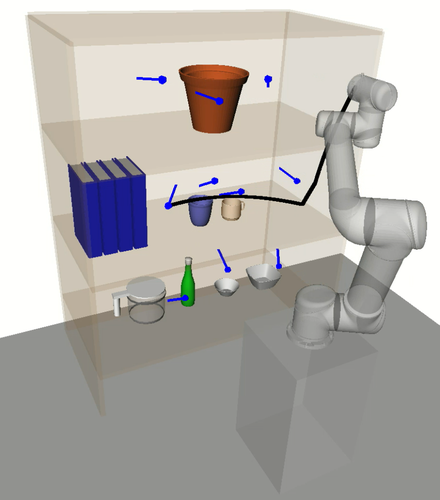}%
\label{fig:ur5_hap_seq10}}
\hfil
\subfloat[$\sigma_{HAP}\lbrack0:9\rbrack, \theta^2 $]{\includegraphics[width=0.15\textwidth]{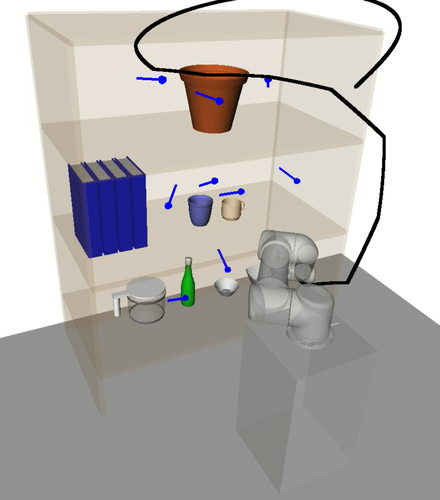}%
\label{fig:ur5_hap_seq11}}
\hfil
\subfloat[$\sigma_{HAP}\lbrack9:10\rbrack, \theta^2 $]{\includegraphics[width=0.15\textwidth]{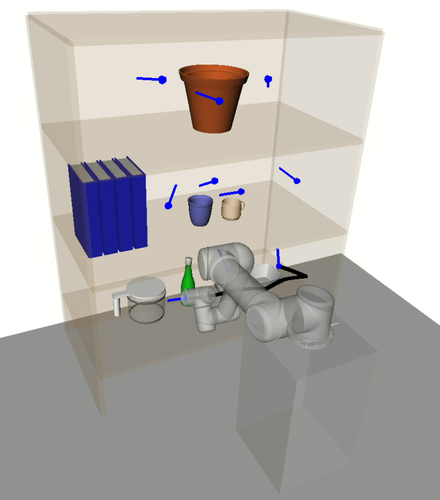}%
\label{fig:ur5_hap_seq12}}
\hfil
\subfloat[$\sigma_{HAP}\lbrack10:0\rbrack, \theta^2 $]{\includegraphics[width=0.15\textwidth]{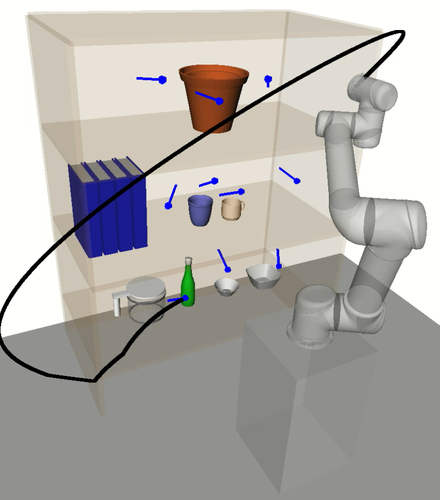}%
\label{fig:ur5_hap_seq13}}
\hfil
\caption{Example HAP trajectory sequence for UR5 bookshelf experiment. The assigned $\epsilon$-GHA indexes for each task pair are shown to highlight when subspace switching occurs. Tasks are all successfully planned to and trajectories between intra-subspace tasks appear to be consistently short and smooth. Tasks requiring a large change around the shoulder joint are grouped into $\theta^2$ and executed together in \protect\subref{fig:ur5_hap_seq11}-\protect\subref{fig:ur5_hap_seq13}.}
\label{fig:ur5_hap_seq}
\end{figure}

In this section we discuss the experimental results and provide an analysis on HAP's performance compared to the baselines. A key benefit of HAP is its ability to effectively reason about low-level motion during task sequencing.
A notable supporting observation of this is that HAP's task sequences tended to minimise movement between bookshelf rows. This is in contrast to the baselines which ignore the cost of avoiding collision with the shelves and objects, leading to frequent switching between rows and resulting in inefficient trajectory sequences and higher planning failure rates. See Appendix~\ref{app:multimedia} for supplementary videos with example experiments.

More specifically to RoboTSP, its decoupled sequencing approach fails to reason about the non-linear relationship between task space and configuration space; that is, short paths in task space are not necessarily so in configuration space. This is evidenced in the execution time plot for the UR5 experiments (Fig.~\ref{fig:ur5_exec_time}) where it is expected that average execution times should decrease as the number of tasks increase due to the increasing spatial density of the tasks; however, the opposite is observed for RoboTSP. An example consequential behaviour is depicted in Fig.~\ref{fig:ur5_robotsp_seq} where unnecessarily long trajectories are executed.

Cluster-RTSP performs overall better than RoboTSP in the UR5 experiments. Compared to HAP, however, trajectory jerks and planning time and success rates were overall worse. A contributor of this, as earlier claimed, is that outlier tasks with highly dissimilar IK solutions can have an overall negative impact on the trajectory sequence, as depicted in Fig.~\ref{fig:ur5_clustertsp_seq}. This can be explained by the initial unique configuration assignment step which attempts to find configurations that are all close to each other.

In contrast, the $\epsilon$-GHAs computed by HAP account for nonlinearities by considering prior knowledge such as the manipulator's kinematic structure and environment obstacles.
This is exemplified in Fig.~\ref{fig:ur5_hap_seq} where trajectories between intra-subspace tasks appear to be consistently short and smooth. Furthermore, tasks requiring a large change in configuration are grouped into the same subspace and executed together without negatively affecting other tasks in the sequence.

Enabling the base of the UR5 robot to be mobile further enhanced performance due to the greater flexibility afforded by the diverse subspace assignments (see Figs.~\ref{fig:ur5_subspaces}-\ref{fig:ur5_mobile_subspaces}). Emergent behaviour such as translating the base away from poses in the lower shelf in order to avoid self-collision were observed. See Appendix~\ref{app:multimedia} for a supplementary video demonstrating this.

While the above effects are less pronounced in the Sawyer experiments due to its kinematic redundancy (see Fig.~\ref{fig:reachable_comparison}), Cluster-RTSP's performance was notably worse than RoboTSP and HAP. This can be explained by the algorithm's inability to effectively reason over the larger number of redundant IK solutions afforded by the arm's extra DOF. Similar findings were reported in the original paper~\cite{wong2020novel}.

Finally, one may suspect that HAP produces similar goal configurations to the baselines and the seed trajectories retrieved from the $\epsilon$-GHA mappings led to higher success rates. However, we elucidate this claim by showing that HAP-no-prior still performs markedly better even without informed trajectory priors.

\section{CONCLUSIONS AND FUTURE WORK}
We have presented a new multi-query task sequencing framework for robotic manipulators that is designed to perform efficiently in practice given a user-defined task space. The framework computes a task-space decomposition that quickly produces efficient task sequences and motion plans online. The decomposition is constructed by finding a set of subspaces with associated $\epsilon$-Gromov-Hausdorff approximations that guarantee short paths of bounded length which also can be concatenated smoothly.

We present theoretical analyses and extensive empirical evaluations. We evaluate our framework with several kinematic configurations over long task sequences in a complex bookshelf environment. Results showed notable performance improvement over state-of-the-art baseline methods in planning time, planning success rate, and smoothness measured by jerk. This highlights the importance of reasoning about the low-level motion of the manipulator during sequencing which HAP facilitates in a computationally efficient manner.

\subsection{Limitations}
Here we discuss limitations of our framework and the current implementation. By reducing the search space of the planner, we intentionally trade off flexibility in the range of solutions for computation speed. However, this may lead to pruning of potentially optimal solutions. One has some control over this by, for example, utilising an increasing number of redundant $\epsilon$-GHAs to cover more of the configuration space. In the limit the full solution space can be recovered.

Furthermore, the $\epsilon$-GHA objective in~\eqref{eqn:mod_dijk_obj} is a combination of task-space coverage and degree-of-path preservation which is not directly related to the task sequencing objective; the total sequence cost. Thus, it is difficult to draw any guarantees on the overall path sequence optimality. One potential work around is to store a larger number of subspaces with varying coverage of the configuration space and select a subset of these based on the online scenario.

The online planning process additionally introduces suboptimality to the overall sequence cost. For example, tasks are assigned greedily to a subspace, subspace groups are executed in arbitrary order, intra-subspace sequences are subject to the suboptimality bounds of the TSP solver and the utilisation of a fixed home configuration for transitioning between subspaces. We chose these as they were straight-forward implementations and in practice performed well in our problem setting. In general, we envision many potential approaches to utilising the $\epsilon$-GHA mappings and our online planning method is just one such approach rather than a fundamental component of the framework.

\subsection{Future Work}

Our results and limitations discussed above motivate several important avenues of future work. For example, formulating the $\epsilon$-GHA objective in a way that directly considers the task sequence objective would be an interesting extension. Regarding the current implementation, HAP consists of several hyperparameters that may be cumbersome to tune. Automatic tuning of these could help improve usability. For example, one could train a learning model to predict $\epsilon$ based on the robot model, task space and environment.

It would be interesting to explore methods that would adapt the subspaces online in response to changes in the environment, potentially using online domain adaptation techniques~\cite{tompkins2020online} and conditional density estimation techniques~\cite{zhi2021trajectory}.
While we focus on task sequencing in this work, $\epsilon$-GHAs are potentially useful for other applications. For example, local reactive controllers such as RMPflow~\cite{cheng2021rmpflow} would benefit from knowledge of which regions of the task space are approximate isomorphisms to configuration space. 

Another interesting avenue of future work is to compute the maps using a continuous representation, for example a Gaussian process, while maintaining the bounded suboptimality guarantees.  Additionally, applying HAP to other embodiments such as high-DOF humanoids presents interesting challenges.

\bibliographystyle{IEEEtran}
\balance 
\bibliography{references}

\appendix

\subsection
{Multimedia}\label{app:multimedia}
This article includes videos showcasing example task sequences generated by our method, HAP, and baselines, RoboTSP and Cluster-RTSP. UR5 bookshelf experiments are shown in the video titled ``Supplementary Video 1". Sawyer bookshelf experiments are shown in the video titled ``Supplementary Video 2".

\subsection{Notation}

\begin{table}[!h] 
\caption{Table of Notation}
\begin{tabular}{ p{1cm} p{7cm} } 
\toprule
  $\mathcal{C}$ & Configuration space of the robot.\\
  $\mathcal{T}$ & Task space of the robot (set of valid end-effector poses in $SE(3)$).\\
  $\hat{T}$ & Discretised task space used to compute $\epsilon$-GHAs.\\
  $t$ & A task (6D pose in $SE(3)$). \\
  $T$ & Set of online tasks.\\
  $A$ & Robot model. \\
  $q$ & A robot configuration. \\
  $Q(t)$ & Set of valid IK solution for task $t$.\\
  $\hat{m}$ & Approximate model of environment obstacles (offline).\\
  $\pi$ & Path of robot (sequence of configurations). \\
  $\Pi$ & Set of all possible paths between two tasks (due to multiple IK solutions per task).\\
  $d_C$ & Configuration space distance metric.\\
  $d_T$ & Task space distance metric.\\
  $\theta$ & Denotes a unique mapping from task to configuration space.\\
  $G$ & Undirected graph used for computing $\epsilon$-GHAs (constructed from $\hat{T}$ and $\hat{E}$ \\
  $\hat{E}$ & Edges formed between nodes $T$ based on neighbouring distance condition, e.g. within specific radius. \\
  $E$ & $\epsilon$-GHAs are subgraphs of $G$ and consist of $\theta$ and corresponding edges $E$. \\
  $T_{open}$ & Subset of tasks in $\hat{T}$ without an assigned mapping in any $\theta^i$. \\
  $g$ & Minimum cost path between two tasks given $\theta$.\\
  $J$ & Optimisation objective for computing $\theta$ \\
  
\bottomrule
\label{table:notation}
\end{tabular}
\end{table}

\subsection{HAP Task-space Decomposition for Bookshelf Experiments}\label{sec:multiple_HAs}

Here, illustrations of HAP's subspace allocation process are shown. In Fig.~\ref{fig:reachable_comparison}, a visualisation of the 7-DOF Sawyer arm's naturally extended task-space coverage compared to the 6-DOF UR5 model is shown. This increased task-space coverage is due to the redundant kinematic configuration of the 7-DOF arm.

In Figs.~\ref{fig:ur5_subspaces}\subref{fig:ur5_subspace1}-\subref{fig:ur5_subspace5} each visualisation shows a new subspace found in an iteration of Alg.~\ref{alg:mod_dijkstra} for the UR5 with base mobility disabled. With each iteration, the overall coverage of $\hat{T}$ is increased.
Overlap between subspaces in task space is beneficial as it provides additional IK solutions to choose from during online planning. Subspaces found for the UR5 with base mobility now enabled are visualised in Figs.~\ref{fig:ur5_mobile_subspaces}\subref{fig:ur5_mobile_subspace1}-\subref{fig:ur5_mobile_subspace5} in the order they are generated by HAP.
Subspace boundaries are not always obvious and would be difficult to generate manually.

\begin{figure*}[b]
\centering
\subfloat[$\hat{T}^0$ for 7-DOF Sawyer]{\includegraphics[width=0.21\textwidth]{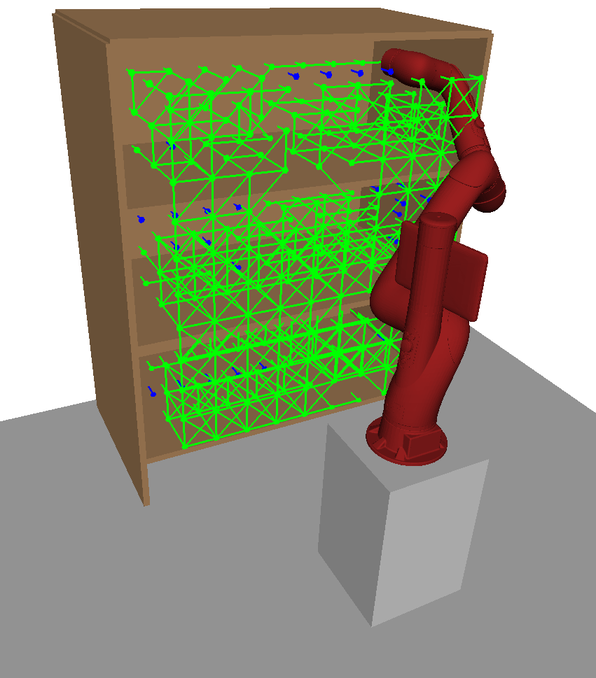}%
\label{fig:sawyer_reachable}}
\hfil
\subfloat[$\hat{T}^0$ for 6-DOF UR5]{\includegraphics[width=0.21\textwidth]{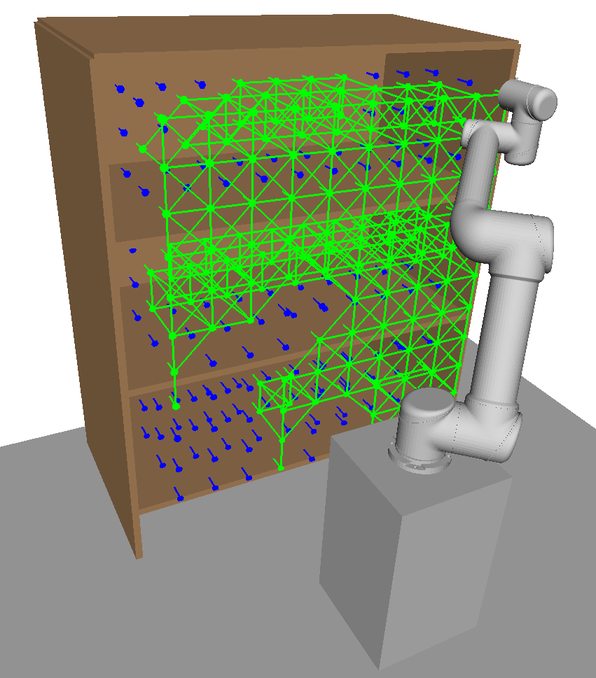}%
\label{fig:ur5_reachable}}
\hfil
\caption{Comparison of first subspaces generated for~\protect\subref{fig:sawyer_reachable} the 7-DOF Sawyer and~\protect\subref{fig:ur5_reachable} the 6-DOF UR5. Green poses and edges indicate regions of task space that lie within the defined subspace; blue poses indicate the remaining unmapped regions. Arms are shown in their corresponding mapped configuration. The 7-DOF Sawyer arm is capable of almost full task-space coverage with a single subspace, in contrast to the UR5, due to its kinematic redundancy.}
\label{fig:reachable_comparison}
\end{figure*}

\begin{figure*}[b]
\centering
\subfloat[$\hat{T}^0$]{\includegraphics[width=0.19\textwidth]{Images/subspaces/ur5/ur5_subspace1_resized.png}%
\label{fig:ur5_subspace1}}
\hfil
\subfloat[$\hat{T}^1$]{\includegraphics[width=0.19\textwidth]{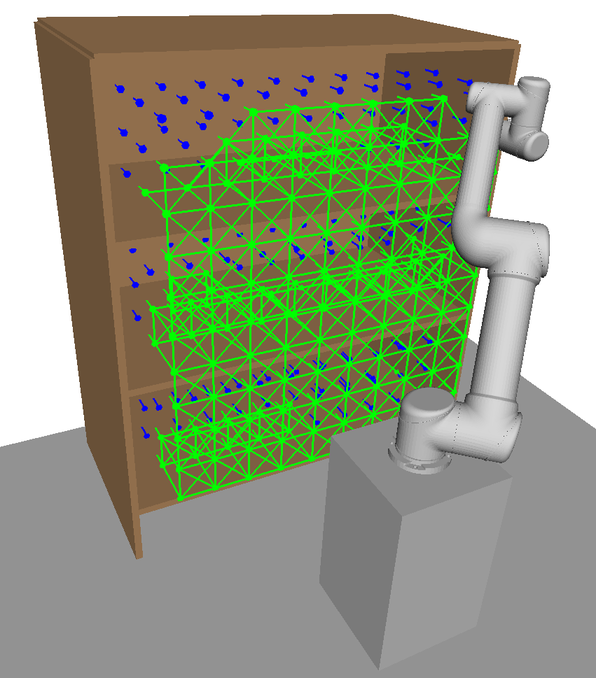}%
\label{fig:ur5_subspace2}}
\hfil
\subfloat[$\hat{T}^2$]{\includegraphics[width=0.19\textwidth]{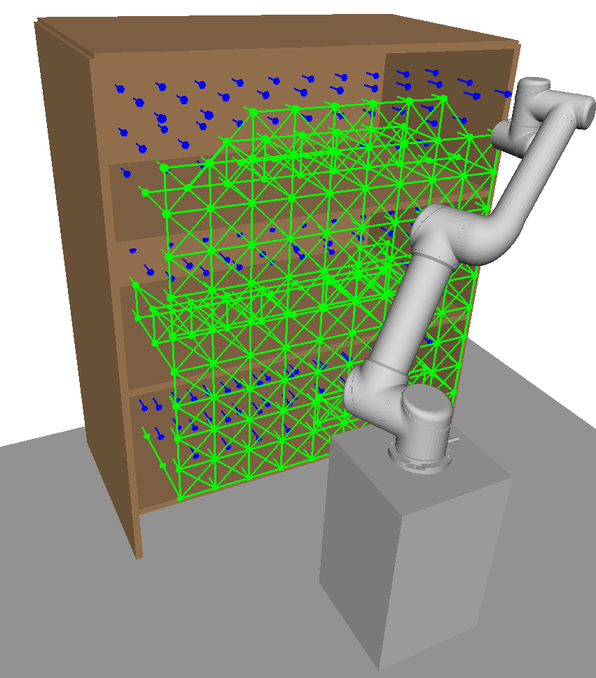}%
\label{fig:ur5_subspace3}}
\hfil
\subfloat[$\hat{T}^3$]{\includegraphics[width=0.19\textwidth]{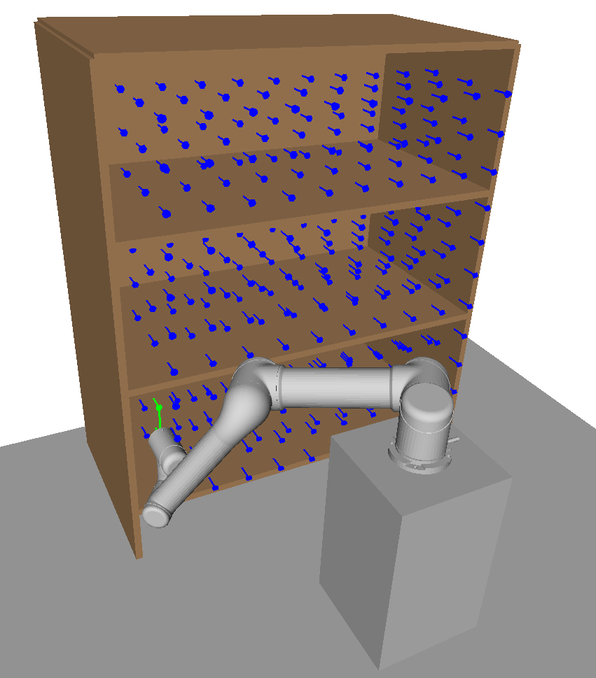}%
\label{fig:ur5_subspace4}}
\hfil
\subfloat[$\hat{T}^4$]{\includegraphics[width=0.19\textwidth]{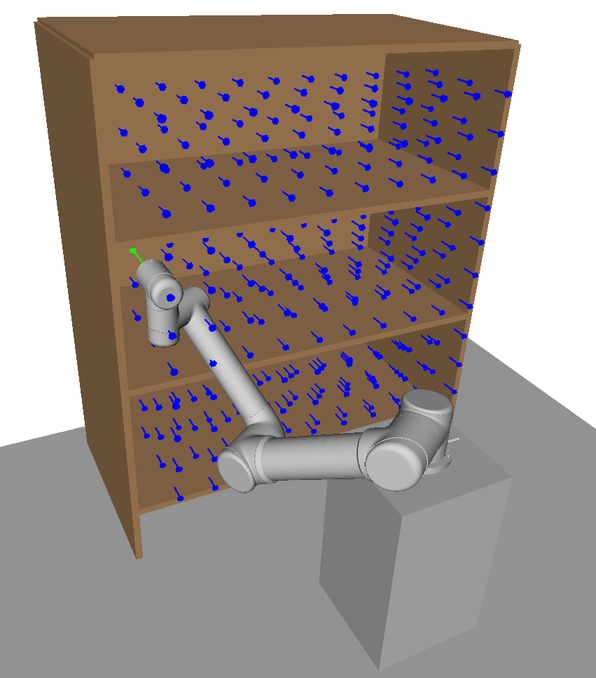}%
\label{fig:ur5_subspace5}}
\hfil
\caption{Visualisation of task-space subspaces for UR5 with static base. Green and blue poses are defined as in the previous figure. Arms are shown in their corresponding mapped configuration. Subspaces are sorted by order in which they were found by Alg.~\ref{alg:gen_database}. Subspaces in \protect\subref{fig:ur5_subspace1}~and~\protect\subref{fig:ur5_subspace2} achieve large and diverse coverage, while the subspace in~\protect\subref{fig:ur5_subspace3} achieves slightly better coverage of the right side of the bookshelf with the shown shoulder orientation~\protect\subref{fig:ur5_subspace2}. Subspaces in~\protect\subref{fig:ur5_subspace4}~and~\protect\subref{fig:ur5_subspace5} cover only small isolated regions of the task space in the left bottom and middle shelf rows, respectively (note that these are removed for the experiments).}
\label{fig:ur5_subspaces}
\end{figure*}

\begin{figure*}[b]
\centering
\subfloat[$\hat{T}^0$]{\includegraphics[width=0.19\textwidth]{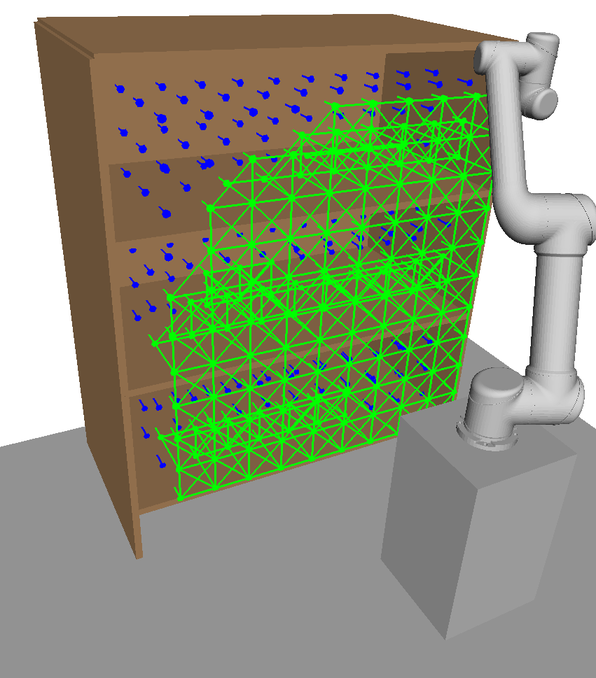}%
\label{fig:ur5_mobile_subspace1}}
\hfil
\subfloat[$\hat{T}^1$]{\includegraphics[width=0.19\textwidth]{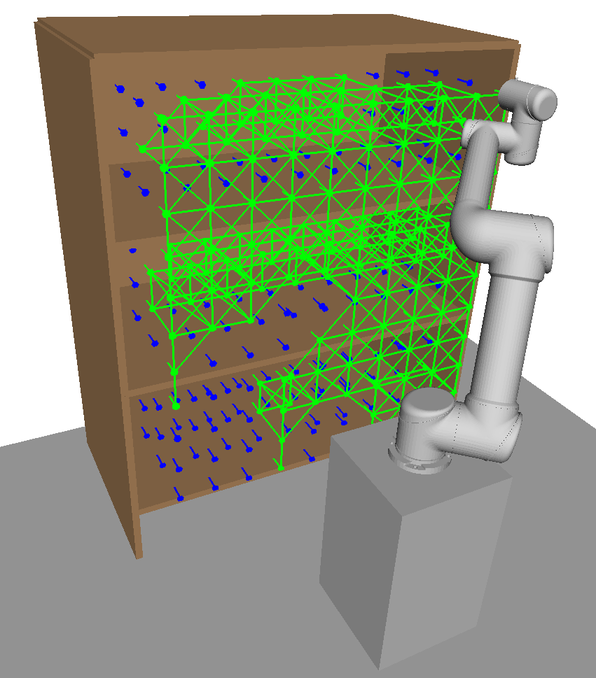}%
\label{fig:ur5_mobile_subspace2}}
\hfil
\subfloat[$\hat{T}^2$]{\includegraphics[width=0.19\textwidth]{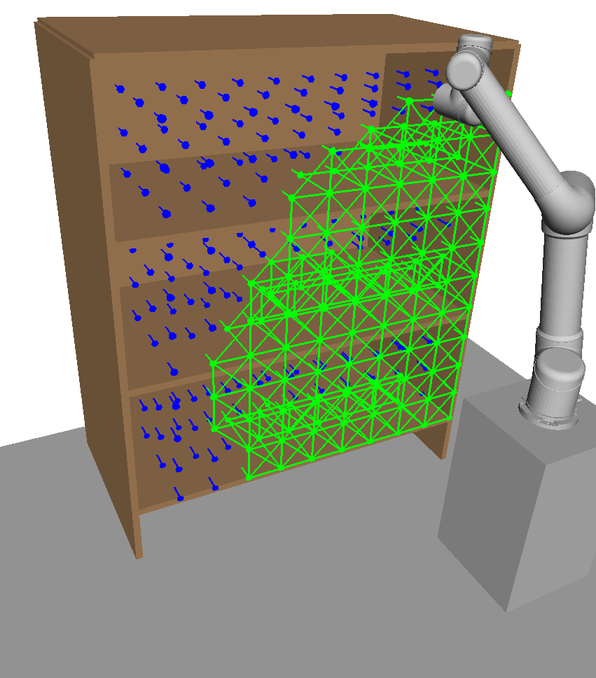}%
\label{fig:ur5_mobile_subspace3}}
\hfil
\subfloat[$\hat{T}^3$]{\includegraphics[width=0.19\textwidth]{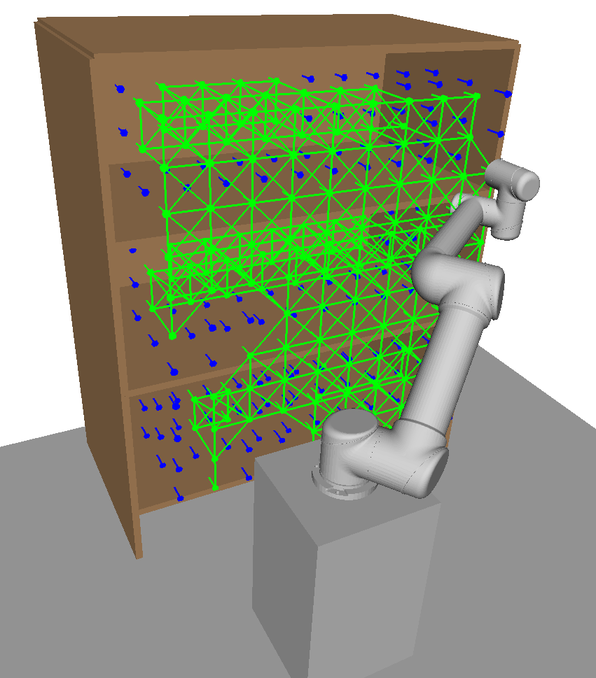}%
\label{fig:ur5_mobile_subspace4}}
\hfil
\subfloat[$\hat{T}^4$]{\includegraphics[width=0.19\textwidth]{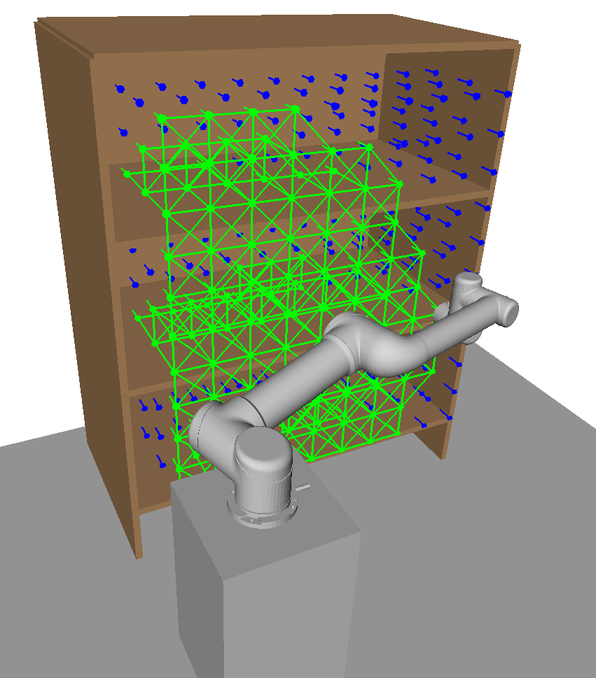}%
\label{fig:ur5_mobile_subspace5}}
\hfil
\caption{Visualisation of task-space subspaces for HAP-Mobile. Note the base pose changes for each subspace. Subspaces are sorted by order in which they were found by Alg.~\ref{alg:gen_database}. Subspaces exhibit large and diverse coverage for all base positions.}
\label{fig:ur5_mobile_subspaces}
\end{figure*}

\end{document}